\documentclass[twoside,11pt]{article}

\usepackage{jmlr2e_modified_no_mention_of_journal} 
\usepackage{times}
\usepackage{algorithm}
\usepackage{algorithmic}
\usepackage{amsfonts}
\usepackage{amsmath}
\usepackage{latexsym}
\usepackage{color}
\usepackage{graphicx}
\usepackage{enumerate}
\usepackage{amstext}
\usepackage{url}

\def\Rset{\mathbb{R}}

\def\S{\mathbb{S}}

\DeclareMathOperator*{\E}{\rm E}

\DeclareMathOperator{\Pdim}{Pdim}
\DeclareMathOperator{\card}{Card}

\newcommand{\ignore}[1]{}

\newcommand{\mat}[1]{{\mathbf #1}}
\renewcommand{\u}{\mat{u}}
\renewcommand{\v}{\mat{v}}
\newcommand{\e}{\epsilon}
\newcommand{\h}{\widehat}

\newcommand{\ssigma}{{\boldsymbol \sigma}}
\renewcommand{\L}{{\cal L}}
\newcommand{\cX}{\mathcal{X}}
\newcommand{\cY}{\mathcal{Y}}
\newcommand{\cZ}{\mathcal{Z}}
\newcommand{\set}[1]{\{#1\}}

\newenvironment{proof*}{\trivlist
\item[\hskip\labelsep{\it\proofname}{.}]}

\jmlrheading{3}{2013}{1-25}{10/10}{10/31}{Corinna Cortes and Spencer Greenberg and Mehryar Mohri}
 \ShortHeadings{Relative Deviation and Generalization
   with Unbounded Loss Functions}{Cortes, Greenberg, and Mohri}
\firstpageno{1}

\begin{document}

\title{Relative Deviation Learning Bounds and \\Generalization with Unbounded Loss Functions}

\author{\name Corinna Cortes \email corinna@google.com \\
       \addr Google Research, 76 Ninth Avenue, New York, NY 10011
        \AND
       \name Spencer Greenberg \email spencerg@cims.nyu.edu \\
       \addr Courant Institute, 251 Mercer Street, New York, NY 10012
        \AND
       \name Mehryar Mohri \email mohri@cims.nyu.edu \\
       \addr Courant Institute and Google Research\\251 Mercer Street, New York, NY 10012
       }
       
\editor{TBD}

\maketitle

\begin{abstract}
  We present an extensive analysis of relative deviation bounds,
  including detailed proofs of two-sided inequalities and their
  implications. We also give detailed proofs of two-sided
  generalization bounds that hold in the general case of unbounded
  loss functions, under the assumption that a moment of the loss is
  bounded. These bounds are useful in the analysis of importance
  weighting and other learning tasks such as unbounded regression.
\end{abstract}


\section{Introduction}

Most generalization bounds in learning theory hold only for bounded
loss functions. This includes standard VC-dimension bounds
\citep{vapnik98}, Rademacher complexity
\citep{KoltchinskiiPanchenko2000,BartlettBoucheronLugosi2002,KoltchinskiiPanchenko2002,BartlettMendelson2002}
or local Rademacher complexity bounds
\citep{Koltchinskii2006,BartlettBousquetMendelson2002}, as well as most other bounds
based on other complexity terms. This assumption is typically
unrelated to the statistical nature of the problem considered but it
is convenient since when the loss functions are uniformly bounded,
standard tools such as Hoeffding's inequality
\citep{Hoeffding63,Azuma67}, McDiarmid's inequality
\citep{McDiarmid89}, or Talagrand's concentration inequality
\citep{talagrand_ineq} apply.

There are however natural learning problems where the boundedness
assumption does not hold. This includes unbounded regression tasks
where the target labels are not uniformly bounded, and a variety of
applications such as sample bias correction
\citep{dudik,huang-nips06,bias,sugiyama-nips2008,bickel-icml07},
domain adaptation \citep{bendavid,blitzer,daume06,jiang-zhai07,nadap,nsmooth},
or the analysis of boosting \citep{DasguptaL03}, where the importance
weighting technique is used \citep{importance}. It is therefore
critical to derive learning guarantees that hold for these scenarios
and the general case of unbounded loss functions.

When the class of functions is unbounded, a single function may take
arbitrarily large values with arbitrarily small probabilities.  This
is probably the main challenge in deriving uniform convergence bounds
for unbounded losses.  This problem can be avoided by assuming the
existence of an envelope, that is a single non-negative function with
a finite expectation lying above the absolute value of the loss of
every function in the hypothesis set
\citep{Dudley84,Pollard84,Dudley87,Pollard89,Haussler92}, an
alternative assumption similar to Hoeffding's inequality based on the
expectation of a hyperbolic function, a quantity similar to the
moment-generating function, is used by \cite{MeirZhang2003}. However,
in many problems, e.g., in the analysis of importance weighting even
for common distributions, there exists no suitable envelope function
\citep{importance}. Instead, the second or some other
$\alpha$th-moment of the loss seems to play a critical role in the
analysis. Thus, instead, we will consider here the assumption that
some $\alpha$th-moment of the loss functions is bounded as in
\cite{vapnik98,vapnik06}.

This paper presents in detail two-sided generalization bounds for
unbounded loss functions under the assumption that some
$\alpha$th-moment of the loss functions, $\alpha > 1$, is
bounded. The proof of these bounds makes use of relative deviation
generalization bounds in binary classification, which we also prove
and discuss in detail. Much of the results and material we present is
not novel and the paper has therefore a survey nature.  However, our
presentation is motivated by the fact that the proofs given in the
past for these generalization bounds were either incorrect or
incomplete. 

We now discuss in more detail prior results and proofs.  One-side
relative deviation bounds were first given by \cite{vapnik98}, later
improved by a constant factor by \cite{AnthonyShawe-Taylor1993}.
These publications and several others have all relied on a lower bound
on the probability that a binomial random variable of $m$ trials
exceeds its expected value when the bias verifies $p >
\frac{1}{m}$. This also later appears in \cite{Vapnik2006} and
implicitly in other publications referring to the relative deviations
bounds of \cite{vapnik98}.  To the best of our knowledge, no actual
proof of this inequality was ever given in the past in the machine
learning literature before our recent work \citep{greenbergmohri2013}.
One attempt was made to prove this lemma in the context of the
analysis of some generalization bounds \citep{Jaeger2005}, but
unfortunately that proof is not sufficient to support the general case
needed for the proof of the relative deviation bound of
\cite{vapnik98}. 

We present the proof of two-sided relative deviation bounds in detail
using the recent results of \cite{greenbergmohri2013}. The two-sided
versions we present, as well as several consequences of these bounds,
appear in \cite{AnthonyBartlett99}. However, we could not find a full
proof of the two-sided bounds in any prior publication. Our
presentation shows that the proof of the other side of the inequality
is not symmetric and cannot be immediately obtained from that of the
first side inequality.  Additionally, this requires another proof
related to the binomial distributions given by
\cite{greenbergmohri2013}.

Relative deviation bounds are very informative guarantees in machine
learning of independent interest, regardless of the key role they play
in the proof of unbounded loss learning bounds.  They lead to sharper
generalization bounds whose right-hand side is expressed as the
interpolation of a $O(1/m)$ term and a $O(1/\sqrt{m})$ term that
admits as a multiplier the empirical error or the generalization
error. In particular, when the empirical error is zero, this leads to
faster rate bounds. We present in detail the proof of this type of
results as well as that of several others of interest
\citep{AnthonyBartlett99}. Let us mention that, in the form presented
by \cite{vapnik98}, relative deviation bounds suffer from a
discontinuity at zero (zero denominator), a problem that also affects
inequalities for the other side and which seems not to have been
rigorously treated by previous work. Our proofs and results explicitly
deal with this issue.

We use relative deviations bounds to give the full proofs of two-sided
generalization bounds for unbounded losses with finite moments of
order $\alpha$, both in the case $1 < \alpha \leq 2$ and the case
$\alpha > 2$. One-sided generalization bounds for unbounded loss
functions were first given by \citet{vapnik98,vapnik06} under the same
assumptions and also using relative deviations.  The one-sided version
of our bounds for the case $1 < \alpha \leq 2$ coincides with that of
\citep{vapnik98,vapnik06} modulo a constant factor, but the proofs
given by Vapnik in both books seem to be incorrect.\footnote{In
  \citep{vapnik98}[p.204-206], statement (5.37) cannot be derived from
  assumption (5.35), contrary to what is claimed by the author, and in
  general it does not hold: the first integral in (5.37) is restricted
  to a sub-domain and is thus smaller than the integral of
  (5.35). Furthermore, the main statement claimed in Section (5.6.2)
  is not valid. In \citep{vapnik06}[p.200-202], the author invokes the
  \emph{Lagrange method} to show the main inequality, but the proof
  steps are not mathematically justified. Even with our best efforts,
  we could not justify some of the steps and strongly believe the
  proof not to be correct. In particular, the way function $z$ is
  concluded to be equal to one over the first interval is suspicious
  and not rigorously justified.} The core component of our proof is
based on a different technique using H\"older's inequality. We also
present some more explicit bounds for the case $1 < \alpha \leq 2$ by
approximating a complex term appearing in these bounds. The one-sided
version of the bounds for the case $\alpha > 2$ are also due to
\citet{vapnik98,vapnik06} with similar questions about the
proofs.\footnote{Several of the comments we made for the case $1 <
  \alpha \leq 2$ hold here as well. In particular, the author's proof
  is not based on clear mathematical justifications.  Some steps seem
  suspicious and are not convincing, even with our best efforts to
  justify them.} In that case as well, we give detailed proofs using
the Cauchy-Schwarz inequality in the most general case where a
positive constant is used in the denominator to avoid the
discontinuity at zero. These learning bounds can be used directly
in the analysis of unbounded loss functions as in the case
of importance weighting \citep{importance}.

The remainder of this paper is organized as follows. In
Section~\ref{sec:preliminaries}, we briefly introduce some definitions
and notation used in the next sections. Section~\ref{sec:relative}
presents in detail relative deviation bounds as well as several of
their consequences. Next, in Section~\ref{sec:unbounded} we present
generalization bounds for unbounded loss functions under the
assumption that the moment of order $\alpha$ is bounded first in the
case $1 < \alpha \leq 2$ (Section~\ref{sec:unbounded1}), then in the
case $\alpha > 2$ (Section~\ref{sec:unbounded2}).

\section{Preliminaries}
\label{sec:preliminaries}

We consider an input space $\cX$ and an output space $\cY$, which in
the particular case of binary classification is $\cY = \set{-1, +1}$
or $\cY = \set{0, 1}$, or a measurable subset of $\Rset$ in
regression. We denote by $D$ a distribution over $\cZ = \cX \times
\cY$. For a sample $S$ of size $m$ drawn from $D^m$, we will denote by
$\h D$ the corresponding empirical distribution, that is the
distribution corresponding to drawing a point from $S$ uniformly at
random. Throughout this paper, $H$ denotes a hypothesis of functions
mapping from $\cX$ to $\cY$. The loss incurred by hypothesis $h \in H$
at $z \in \cZ$ is denoted by $L(h, z)$. $L$ is assumed to be
non-negative, but not necessarily bounded. We denote by $\L(h)$ the
expected loss or generalization error of a hypothesis $h \in H$ and by
$\h \L_S(h)$ its empirical loss for a sample $S$:
\begin{equation}
\L(h) = \E_{z \sim D}[L(h, z)] 
\qquad \qquad \h \L_S(h) = \E_{z \sim \h D}[L(h, z)].
\end{equation}
For any $\alpha > 0$, we also use the notation $\L_\alpha(h) = \E_{z
  \sim D}[L^\alpha(h, z)]$ and $\h \L_\alpha(h) = \E_{z \sim \h
  D}[L^\alpha(h, z)]$ for the $\alpha$th moments of the loss. When the
loss $L$ coincides with the standard zero-one loss used in binary
classification, we equivalently use the following notation
\begin{equation}
R(h) = \E_{z = (x, y) \sim D}[1_{h(x) \neq y}] 
\qquad \qquad \h R_S(h) = \E_{z = (x, y) \sim \h D}[1_{h(x) \neq y}] .
\end{equation}
We will sometimes use the shorthand $x_1^m$ to denote a sample of $m >
0$ points $(x_1, \ldots, x_m) \in \cX^m$.  For any hypothesis set $H$
of functions mapping $\cX$ to $\cY = \set{-1, +1}$ or $\cY = \set{0,
  1}$ and sample $x_1^m$, we denote by $\S_H(x_1^m)$ the number of
distinct dichotomies generated by $H$ over that sample and by
$\Pi_m(H)$ the growth function:
\begin{align}
& \S_H(x_1^m) = \card \Big(\big\{(h(x_1), \ldots, h(x_{m})) \colon h
  \in H \big\} \Big)\\
& \Pi_m(H) = \max_{x_1^m \in \cX^m} \S_H(x_1^m).
\end{align}

\section{Relative deviation bounds}
\label{sec:relative}

In this section we prove a series of relative deviation learning
bounds which we use in the next section for deriving generalization
bounds for unbounded loss functions. We will assume throughout the
paper, as is common in much of learning theory, that each expression
of the form $\sup_{h \in H} [. . .]$ is a measurable function, which
is not guaranteed when $H$ is not a countable set. This assumption
holds nevertheless in most common applications of machine learning.

We start with the proof of a symmetrization lemma
(Lemma~\ref{lemma:lemma1}) originally presented by \cite{vapnik98},
which is used by \cite{AnthonyShawe-Taylor1993}.  These publications
and several others have all relied on a lower bound on the probability
that a binomial random variable of $m$ trials exceeds its expected
value when the bias verifies $p > \frac{1}{m}$.  To our knowledge, no
rigorous proof of this fact was ever provided in the literature in the
full generality needed. The proof of this result was recently given
by \cite{greenbergmohri2013}.

\begin{lemma}[\cite{greenbergmohri2013}]
\label{lemma:binomial}
Let $X$ be a random variable distributed according to the binomial
distribution $B(m, p)$ with $m$ a positive integer (the number of trials) and $p >
\frac{1}{m}$ (the probability of success of each trial).  Then, the following inequality holds:
\begin{equation}
\label{eq:main}
\Pr\Big[X \geq  \E[X]\Big] > \frac{1}{4}, 
\end{equation}
where $\E[X] = mp$. 
\end{lemma}
The lower bound is never reached but is approached asymptotically when
$m = 2$ as $p \to \frac{1}{2}$ from the right.

Our proof of Lemma~\ref{lemma:lemma1} is more concise than that of
\cite{vapnik98}. Furthermore, our statement and proof handle the
technical problem of discontinuity at zero ignored by previous
authors. The denominator may in general become zero, which would lead
to an undefined result. We resolve this issue by including an
arbitrary positive constant $\tau$ in the denominator in most of our
expressions.

For the proof of the following result, we will use the function $F$
defined over $(0, +\infty) \times (0, +\infty)$ by $F\colon (x, y)
\mapsto \frac{x - y}{\sqrt[\alpha]{\frac{1}{2}[x + y +
    \frac{1}{m}]}}$.  By Lemma~\ref{lemma:f}, $F(x,y)$ is increasing
in $x$ and decreasing in $y$.

\begin{lemma}
\label{lemma:lemma1}
Let $1 < \alpha \leq 2$. Assume that $m \e^{\frac{\alpha}{\alpha - 1}}
> 1$. Then, for any hypothesis set $H$ and any $\tau > 0$, the
following holds:
\begin{equation*}
\Pr_{S \sim D^m} \bigg[\sup_{h \in H} \frac{R(h) - \h R_S(h)}{\sqrt[\alpha]{R(h)+\tau}} > 
\e \bigg]
\leq 4 \Pr_{S, S' \sim D^m} \bigg[\sup_{h \in H} \frac{ \h R_{S'}(h) - \h R_{S}(h) }{
  \sqrt[\alpha]{\frac{1}{2} [\h R_S(h) + \h R_{S'}(h) + \frac{1}{m}] } } > \e \bigg].
\end{equation*}
\end{lemma}
\begin{proof}
  We give a concise version of the proof given by \citep{vapnik98}. We
  first show that the following implication holds for any $h \in H$: 
\begin{equation}
\label{eq:implication}
\left( \frac{R(h) - \h
    R_S(h)}{\sqrt[\alpha]{R(h) + \tau}} > \e \right) \wedge \left( \h R_{S'}(h) >
  R(h) \right) \Rightarrow F(\h R_{S'}(h), \h R_{S}(h)) > \e.
\end{equation}
The first condition can be equivalently rewritten as $\h R_S(h) < R(h)
- \e (R(h) + \tau)^{\frac{1}{\alpha}}$, which implies
\begin{equation}   
\label{eq:h0}
\h R_S(h) < R(h) - \e R(h)^{\frac{1}{\alpha}}\\
\qquad \text{and} \qquad \e^{\frac{\alpha}{\alpha - 1}} < R(h),
\end{equation}
since $\h R_S(h) \ge 0$. Assume that the antecedent of the implication
\eqref{eq:implication} holds for $h \in H$. Then, in view of the
monotonicity properties of function $F$ (Lemma~\ref{lemma:f}), we can
write:
\begin{align*} 
F(\h R_{S'}(h), \h R_{S}(h))
& \geq F(R(h), R(h) - \e R(h)^{\frac{1}{\alpha}}) & \text{($\h R_{S'}(h) >
  R(h)$ and 1st ineq. of \eqref{eq:h0})}\\ 
& = \frac{R(h) - (R(h) - \e 
  R(h)^{\frac{1}{\alpha}})}{\sqrt[\alpha]{\frac{1}{2}[2 R(h) - \e 
  R(h)^{\frac{1}{\alpha}} + \frac{1}{m}]}} \\
& \geq  \frac{\e R(h)^{\frac{1}{\alpha}}}{\sqrt[\alpha]{\frac{1}{2}[2 R(h) -
  \e^{\frac{\alpha}{\alpha - 1}} + \frac{1}{m}]}}  & \text{(2nd ineq. of \eqref{eq:h0})}\\
& > \frac{\e R(h)^{\frac{1}{\alpha}}}{\sqrt[\alpha]{\frac{1}{2}[2 R(h)
  ]}} = \e,  & \text{($m \e^{\frac{\alpha}{\alpha - 1}} > 1$)}
\end{align*}
which proves \eqref{eq:implication}. Now, by definition of the
supremum, for any $\eta > 0$, there exists $h_0 \in H$ such
that
\begin{equation}
\label{eq:sup}
  \sup_{h \in H}  \frac{R(h) - \h R_S(h)}{\sqrt[\alpha]{R(h) + \tau}} -
  \frac{R(h_0) - \h R_S(h_0)}{\sqrt[\alpha]{R(h_0) + \tau}}   \le \eta.
\end{equation}
Using the definition of $h_0$ and implication \eqref{eq:implication}, we can write
\begin{align*}
& \Pr_{S, S' \sim D^m} \bigg[\sup_{h \in H} \frac{ \h R_{S'}(h) - \h R_{S}(h) }{
  \sqrt[\alpha]{\frac{1}{2} [\h R_S(h) + \h R_{S'}(h)  + \frac{1}{m}]}
} > \e \bigg] \\
& \geq \Pr_{S, S' \sim D^m} \bigg[\frac{ \h R_{S'}(h_0) - \h R_{S}(h_0) }{
  \sqrt[\alpha]{\frac{1}{2} [\h R_S(h_0) + \h R_{S'}(h_0)  + \frac{1}{m}]}
} > \e \bigg] & \text{(by def. of $\sup$)}\\
& \geq \Pr_{S, S' \sim D^m} \left[ \bigg ( \frac{R(h_0) - \h R_S(h_0)}{\sqrt[\alpha]{R(h_0)+\tau}} > 
\e \bigg) \wedge \Big( R_{S'}(h_0) > R(h_0) \Big) \right] &
(\text{implication \eqref{eq:implication}})\\
& = \Pr_{S \sim D^m} \left[ \frac{R(h_0) - \h R_S(h_0)}{\sqrt[\alpha]{R(h_0)+\tau}} > 
\e \right] \Pr_{S' \sim D^m} \left[ R_{S'}(h_0) > R(h_0) \right] & \text{(independence)}.
\end{align*}
We now show that this implies the following inequality
\begin{equation}
\label{eq:9}
\Pr_{S, S' \sim D^m}
\bigg[\sup_{h \in H} \frac{ \h R_{S'}(h) - \h R_{S}(h) }{
  \sqrt[\alpha]{\frac{1}{2} [\h R_S(h) + \h R_{S'}(h) + \frac{1}{m}]}
} > \e \bigg] 
\geq \frac{1}{4} \Pr_{S \sim D^m} \left[ \sup_{h \in H}
  \frac{R(h) - \h R_S(h)}{\sqrt[\alpha]{R(h)+\tau}} > \e + \eta
\right],
\end{equation}
by distinguishing two cases. If $R(h_0) > \e^{\frac{\alpha}{\alpha -
    1}}$, since $\e^{\frac{\alpha}{\alpha - 1}} > \frac{1}{m}$, by
Theorem~\ref{lemma:binomial} the inequality $\Pr_{S' \sim D^m} \left[
  R_{S'}(h_0) > R(h_0) \right] > \frac{1}{4}$ holds, which yields
immediately \eqref{eq:9}. Otherwise we have $R(h_0) \leq
\e^{\frac{\alpha}{\alpha - 1}}$. Then, by \eqref{eq:h0}, the condition
$\frac{R(h_0) - \h R_S(h_0)}{\sqrt[\alpha]{R(h_0)+\tau}} > \e$ cannot
hold for any sample $S \sim D^m$ which by \eqref{eq:sup} implies that
the condition $\sup_{h \in H} \frac{R(h) - \h
  R_S(h)}{\sqrt[\alpha]{R(h) + \tau}} > \e + \eta$ cannot hold for any
sample $S \sim D^m$, in which case \eqref{eq:9} trivially holds.  Now,
since \eqref{eq:9} holds for all $\eta > 0$, we can take the limit
$\eta \to 0$ and use the right-continuity of the cumulative
distribution to obtain
\begin{equation*}
\Pr_{S, S' \sim D^m} \bigg[\sup_{h \in H} \frac{ \h R_{S'}(h) - \h R_{S}(h) }{
  \sqrt[\alpha]{\frac{1}{2} [\h R_S(h) + \h R_{S'}(h)  + \frac{1}{m}]}
} > \e \bigg] 
\geq \frac{1}{4} \Pr_{S \sim D^m} \left[ \sup_{h \in H}  \frac{R(h)
    - \h R_S(h)}{\sqrt[\alpha]{R(h) + \tau}} > 
\e \right],
\end{equation*}
which completes the proof of Lemma~\ref{lemma:lemma1}.
\end{proof}
Note that the factor of 4 in the statement of lemma~\ref{lemma:lemma1}
can be modestly improved by changing the condition assumed from
$\e^{\frac{\alpha}{\alpha - 1}}>\frac{1}{m}$ to
$\e^{\frac{\alpha}{\alpha - 1}}>\frac{k}{m}$ for constant values of $k
> 1$.  This leads to a slightly better lower bound on $\Pr_{S' \sim
  D^m} \left[ R_{S'}(h_0) > R(h_0) \right]$, e.g.\ $3.375$ rather than
$4$ for $k = 2$, at the expense of not covering cases where the number
of samples $m$ is less than $\frac{k}{\e^{\frac{\alpha}{\alpha -
      1}}}$. For some values of $k$, e.g.\ $k = 2$, covering these
cases is not needed for the proof of our main theorem
(Theorem~\ref{th:relative}) though.  However, this does not seem to
simplify the critical task of proving a lower bound on $\Pr_{S' \sim
  D^m} \left[ R_{S'}(h_0) > R(h_0) \right]$, that is the probability
that a binomial random variable $B(m, p)$ exceeds its expected value
when $p > \frac{k}{m}$.  One might hope that restricting the range of
$p$ in this way would help simplify the proof of a lower bound on the
probability of a binomial exceeding its expected value. Unfortunately,
our analysis of this problem and proof \citep{greenbergmohri2013}
suggest that this is not the case since the regime where $p$ is small
seems to be the easiest one to analyze for this problem.

\begin{figure}
\centering
\includegraphics[scale=.41]{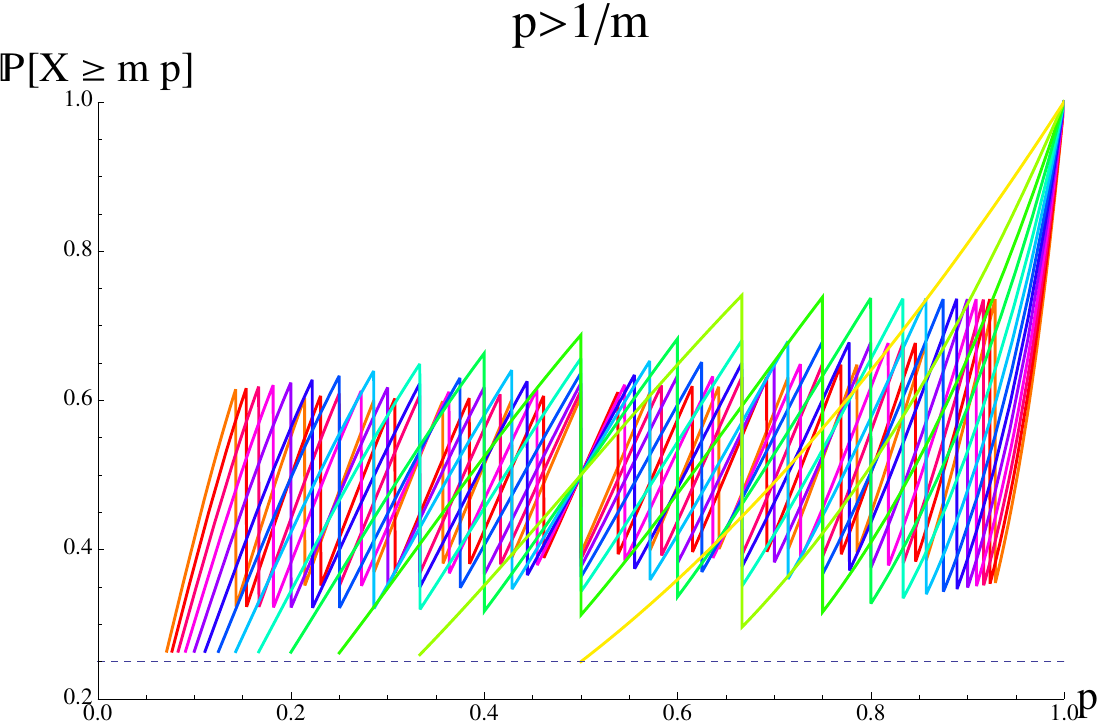}
\includegraphics[scale=.41]{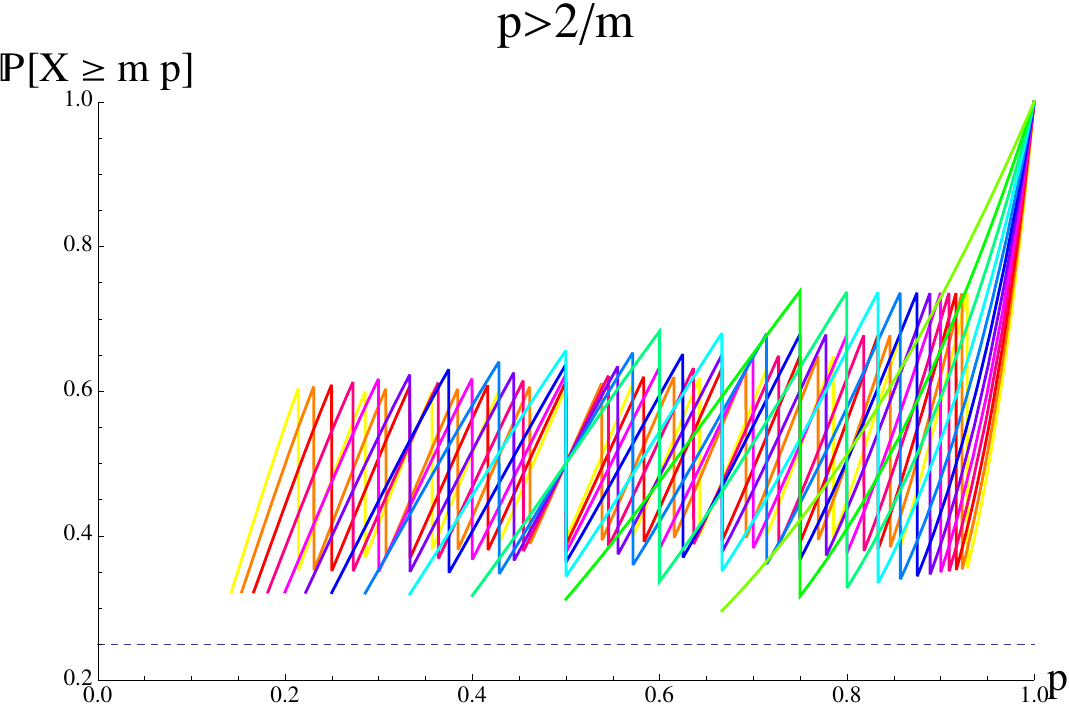}
\caption{These plots depict $\Pr[X \geq \E[X]]$, the probability that
  a binomially distributed random variable $X$ exceeds its
  expectation, as a function of the trial success probability
  $p$. The left plot shows only regions satisfying $p>\frac{1}{m}$ whereas
  the right plot shows only regions satisfying $p>\frac{2}{m}$.
  Each colored line corresponds to a different number of trials,
  $m = 2, 3, \ldots, 14$. The dashed horizontal line at
  $\frac{1}{4}$ represents the value of the lower bound used in the proof of lemma~\ref{lemma:lemma1}.}
\label{fig:campPaulson}
\end{figure}

The result of Lemma~\ref{lemma:lemma1} is a one-sided inequality.  The
proof of a similar result (Lemma~\ref{lemma:lemma2}) with the roles of
$R(h)$ and $\h R_S(h)$ interchanged makes use of the following
theorem.

\begin{lemma}[\cite{greenbergmohri2013}]
\label{lemma:binomial2}
Let $X$ be a random variable distributed according to the binomial
distribution $B(m, p)$ with $m$ a positive integer and $p < 1 -
\frac{1}{m}$.  Then, the following inequality holds:
\begin{equation}
\label{eq:main}
\Pr\Big[X \leq  \E[X] \Big] > \frac{1}{4}, 
\end{equation}
where $\E[X] = mp$. 
\end{lemma}

The proof of the following lemma (Lemma~\ref{lemma:lemma2}) is
novel.\footnote{A version of this lemma is stated in
  \citep{BoucheronBousquetLugosi2005}, but no proof is given.} While
the general strategy of the proof is similar to that of
Lemma~\ref{lemma:lemma1}, there are some non-trivial differences due
to the requirement $p < 1 - \frac{1}{m}$ of
Theorem~\ref{lemma:binomial2}. The proof is not symmetric as shown by
the details given below.

\begin{lemma}
\label{lemma:lemma2}
Let $1 < \alpha \leq 2$. Assume that $m \e^{\frac{\alpha}{\alpha - 1}}
> 1$. Then, for any hypothesis set $H$ and any $\tau > 0$ the
following holds:
\begin{equation*}
\Pr_{S \sim D^m} \bigg[\sup_{h \in H} \frac{\h R_S(h) -  R(h)}{\sqrt[\alpha]{\h R_S(h)+\tau}} > 
\e \bigg]
\leq 4 \Pr_{S, S' \sim D^m} \bigg[\sup_{h \in H} \frac{ \h R_{S'}(h) - \h R_{S}(h) }{
  \sqrt[\alpha]{\frac{1}{2} [\h R_S(h) + \h R_{S'}(h) + \frac{1}{m}] } } > \e \bigg]
\end{equation*}
\end{lemma}

\begin{proof}
  Proceeding in a way similar to the proof of Lemma~\ref{lemma:lemma1}, we first
  show that the following implication holds for any $h \in H$:
\begin{equation}
\label{eq:implication2}
\left( \frac{\h R_S(h) - R(h)}{\sqrt[\alpha]{\h R_S(h) + \tau}} > \e
\right) \wedge \left(  R(h) \geq \h R_{S'}(h) \right) \Rightarrow F(\h R_S(h), \h R_{S'} (h)) > \e.
\end{equation}
The first condition can be equivalently rewritten as $R(h) < \h R_S(h)
- \e (\h R_S(h) + \tau)^{\frac{1}{\alpha}}$, which implies
\begin{equation}   
\label{eq:h1}
R(h) < \h R_S(h) - \e \h R_S(h)^{\frac{1}{\alpha}}\\
\qquad \text{and} \qquad \e^{\frac{\alpha}{\alpha - 1}} < \h R_S(h),
\end{equation}
since $\h R_S(h) \ge 0$. Assume that the antecedent of the implication
\eqref{eq:implication2} holds for $h \in H$. Then, in view of the
monotonicity properties of function $F$ (Lemma~\ref{lemma:f}), we can
write:
\begin{align*} 
F(\h R_{S}(h), \h R_{S'}(h))
& \geq F(\h R_{S}(h), R(h)) & \text{($ R(h) \geq \h R_{S'}(h) $)}\\
& \geq F(\h R_{S}(h), \h R_S(h) - \e \h R_S(h)^{\frac{1}{\alpha}}) & \text{(1st ineq. of \eqref{eq:h1})}\\ 
& = \frac{\h R_S(h) - (\h R_S(h) - \e 
  \h R_S(h)^{\frac{1}{\alpha}})}{\sqrt[\alpha]{\frac{1}{2}[2 \h R_S(h) - \e 
  \h R_S(h)^{\frac{1}{\alpha}} + \frac{1}{m}]}} \\
& \geq  \frac{\e R(h)^{\frac{1}{\alpha}}}{\sqrt[\alpha]{\frac{1}{2}[2 R(h) -
  \e^{\frac{\alpha}{\alpha - 1}} + \frac{1}{m}]}}  & \text{(2nd ineq. of \eqref{eq:h1})}\\
& > \frac{\e R(h)^{\frac{1}{\alpha}}}{\sqrt[\alpha]{\frac{1}{2}[2 R(h)
  ]}} = \e,  & \text{($m \e^{\frac{\alpha}{\alpha - 1}} > 1$)}
\end{align*}
which proves \eqref{eq:implication2}. For the application of
Theorem~\ref{lemma:binomial2} to a hypothesis $h$, the condition $R(h) <
1 - \frac{1}{m}$ is required. Observe that this is implied by the
assumptions $\h R_S(h) \ge \e^{\frac{\alpha}{\alpha - 1}}$ and $m
\e^{\frac{\alpha}{\alpha-1}} > 1$:
\begin{equation*}
R(h) < \h R_S(h) - \e \sqrt[\alpha]{\h R_S(h)} 
\leq 1 - \e \hspace{2pt} \e^{\frac{1}{\alpha - 1}}
= 1 - \e^{\frac{\alpha}{\alpha - 1}} 
< 1- \frac{1}{m}.
\end{equation*}
The rest of the proof proceeds nearly identically to that of
Lemma~\ref{lemma:lemma1}.
\end{proof}

In the statements of all the following results, the term $\E_{x_1^{2m}
  \sim D^{2m}}[\S_H(x_1^{2m})]$ can be replaced by the upper bound
$\Pi_{2m}(H)$ to derive simpler expressions.  By Sauer's lemma
\citep{Sauer72,VapnikChervonenkis71}, the VC-dimension $d$ of the
family $H$ can be further used to bound these quantities since
$\Pi_{2m}(H) \leq \big(\frac{2e m}{d} \big)^d$ for $d \leq 2m$.  The
first inequality of the following theorem was originally stated and
proven by \cite{vapnik98,vapnik06}, later by
\cite{AnthonyShawe-Taylor1993} (in the special case $\alpha = 2$) with
a somewhat more favorable constant, in both cases modulo the
incomplete proof of the symmetrization and the
technical issue related to the denominator taking the value zero, as
already pointed out. The second inequality of the theorem and its
proof are novel. Our proofs benefit from the improved analysis of
\cite{AnthonyShawe-Taylor1993}.

\begin{theorem}
\label{th:relative}
For any hypothesis set $H$ of functions mapping a set $\cX$ to
$\set{0, 1}$, and any fixed $1 < \alpha \leq 2$ and $\tau > 0$, the
following two inequalities hold:
\begin{align*}
  & \Pr_{S \sim D^m} \bigg[\sup_{h \in H} \frac{R(h) - \h
    R_S(h)}{\sqrt[\alpha]{R(h) + \tau}} > \e \bigg] \leq 4 \,
  \E[\S_H(x_1^{2m})] \exp \bigg( \frac{-m^{\frac{2 (\alpha -
        1)}{\alpha}} \e^2 }{2^{\frac{\alpha + 2}{\alpha}} }  \bigg) \\
  & \Pr_{S \sim D^m} \bigg[\sup_{h \in H} \frac{\h R_S(h) -
    R(h)}{\sqrt[\alpha]{\h R_S(h) + \tau}} > \e \bigg] \leq 4 \,
  \E[\S_H(x_1^{2m})] \exp \bigg( \frac{-m^{\frac{2 (\alpha -
        1)}{\alpha}} \e^2 }{2^{\frac{\alpha + 2}{\alpha}} } \bigg).
\end{align*}
\end{theorem}

\begin{proof}
  We first consider the case where $m \e^{\frac{\alpha}{\alpha - 1}}
  \le 1$, which is not covered by Lemma~\ref{lemma:lemma1}. We can
  then write
\begin{align*}
4 \E[\S_H(x_1^{2m})] \exp \bigg[   \frac{-m^{\frac{2 (\alpha -
      1)}{\alpha}} \e^2 }{2^{\frac{\alpha + 2}{\alpha}} }  \bigg]
\ge 
4 \E[\S_H(x_1^{2m})] 
\exp \bigg[   \frac{-1}{2^{\frac{\alpha + 2}{\alpha}} }  \bigg] > 1,
\end{align*} 
for $1 < \alpha \le 2$. Thus, the bounds of the theorem hold trivially
in that case.  On the other hand, when $m \e^{\frac{\alpha}{\alpha -
    1}} \ge 1$, we can apply Lemma~\ref{lemma:lemma1} and
Lemma~\ref{lemma:lemma2}. Therefore, to prove
theorem~\ref{th:relative}, it is sufficient to work with the
symmetrized expression $\sup_{h \in H} \frac{ \h R_{S'}(h) - \h
  R_{S}(h) }{ \sqrt[\alpha]{\frac{1}{2} [\h R_S(h) + \h R_{S'}(h) +
    \frac{1}{m}]} }$, rather than working directly with our original
expressions $\sup_{h \in H} \frac{R(h) - \h R_S(h)}{\sqrt[\alpha]{R(h)
    + \tau}}$ and $\sup_{h \in H} \frac{\h R(h) -
  R_S(h)}{\sqrt[\alpha]{\h R(h) + \tau}}$. To upper bound the
probability that the symmetrized expression is larger than $\e$, we
begin by introducing a vector of Rademacher random variables $\sigma =
(\sigma_{1}, \sigma_{2}, \hdots, \sigma_{m})$, where the $\sigma_{i}$
are independent, identically distributed random variables each equally
likely to take the value $+1$ or $-1$. 
Using the shorthand $x_1^{2m}$ for $(x_1, \ldots, x_{2m})$,
we can then write
\begin{multline*}
\Pr_{S, S' \sim D^m} \bigg[\sup_{h \in H} \frac{ \h R_{S'}(h) - \h R_{S}(h) }{
  \sqrt[\alpha]{\frac{1}{2} [\h R_S(h) + \h R_{S'}(h) + \frac{1}{m}]} } > \e \bigg]\\
\begin{aligned}
& = \Pr_{x_1^{2m} \sim D^{2m}} \bigg[\sup_{h \in H} \frac{ \frac{1}{m}\sum_{i = 1}^m
  (h(x_{m + i}) - h(x_i)) }{
  \sqrt[\alpha]{\frac{1}{2m} [\sum_{i = 1}^m
  (h(x_{m + i}) + h(x_i)) + 1]} } > \e \bigg]\\
& = \Pr_{x_1^{2m} \sim D^{2m}, \ssigma} \bigg[\sup_{h \in H} \frac{
  \frac{1}{m} \sum_{i = 1}^m
  \sigma_i  (h(x_{m + i}) - h(x_i)) }{
  \sqrt[\alpha]{\frac{1}{2m} [\sum_{i = 1}^m
  (h(x_{m + i}) + h(x_i)) + 1]} } > \e \bigg]\\
& = \E_{x_1^{2m} \sim D^{2m}} \bigg[ \Pr_{\ssigma} \bigg[\sup_{h \in H} \frac{\frac{1}{m} \sum_{i = 1}^m
  \sigma_i  (h(x_{m + i}) - h(x_i)) }{
  \sqrt[\alpha]{\frac{1}{2m} [\sum_{i = 1}^m
  (h(x_{m + i}) + h(x_i))+1]} } > \e \, \bigg| \, x_1^{2m} \bigg]\bigg].
\end{aligned}
\end{multline*}
Now, for a fixed $x_1^{2m}$, we have $\E_{\ssigma}\bigg[
\frac{ \frac{1}{m} \sum_{i = 1}^m \sigma_i (h(x_{m + i}) - h(x_i)) }{
  \sqrt[\alpha]{\frac{1}{2m} [\sum_{i = 1}^m (h(x_{m + i}) + h(x_i))+1]} } \bigg]
= 0$, thus, by Hoeffding's inequality, we can write
\begin{multline*}
\Pr_{\ssigma} \left[\frac{\frac{1}{m} \sum_{i = 1}^m
  \sigma_i  (h(x_{m + i}) - h(x_i)) }{
  \sqrt[\alpha]{\frac{1}{2m} [\sum_{i = 1}^m
  (h(x_{m + i}) + h(x_i))]} } > \e \, \bigg| \, x_1^{2m} \right] \\
\begin{aligned}
&\leq
\exp \bigg(   \frac{-[\sum_{i = 1}^{2m}  (h(x_{m + i}) + h(x_i)) + 1]^{\frac{2}{\alpha}}
  m^{\frac{2 (\alpha - 1)}{\alpha}}\e^2 }{2^{\frac{\alpha + 2}{\alpha}} \sum_{i = 1}^m (h(x_{m + i}) - h(x_i))^2}  \bigg)  \\
&\leq
\exp \bigg(   \frac{-[\sum_{i = 1}^{2m}  (h(x_{m + i}) + h(x_i))]^{\frac{2}{\alpha}}
  m^{\frac{2 (\alpha - 1)}{\alpha}}\e^2 }{2^{\frac{\alpha + 2}{\alpha}} \sum_{i = 1}^m (h(x_{m + i}) - h(x_i))^2}  \bigg).
\end{aligned}
\end{multline*}
Since the variables $h(x_i)$, $i \in [1, 2m]$, take values in $\set{0,
  1}$, we can write
\begin{align*}
\sum_{i = 1}^m (h(x_{m + i}) - h(x_i))^2 
& = \sum_{i = 1}^m h(x_{m + i}) + h(x_i) -2 h(x_{m + i})h(x_i)\\
& \leq \sum_{i = 1}^m h(x_{m + i}) + h(x_i)
\leq \Big[ \sum_{i = 1}^m h(x_{m + i}) + h(x_i) \Big]^{\frac{2}{\alpha}},
\end{align*}
where the last inequality holds since $\alpha \le 2$ and the sum is
either zero or greater than or equal to one.  In view of this
identity, we can write
\begin{equation*}
\Pr_{\ssigma} \left[ \frac{ \frac{1}{m} \sum_{i = 1}^m
  \sigma_i  (h(x_{m + i}) - h(x_i)) }{
  \sqrt[\alpha]{\frac{1}{2m} [ \sum_{i = 1}^m
  (h(x_{m + i}) + h(x_i))]} } > \e \, \bigg| \, x_1^{2m} \right] \leq
\exp \bigg(   \frac{-m^{\frac{2 (\alpha -
      1)}{\alpha}} \e^2 }{2^{\frac{\alpha + 2}{\alpha}} }  \bigg).
\end{equation*}
We note now that the supremum over $h \in H$ in the left-hand side
expression in the statement of our theorem need not be over all
hypothesis in $H$: without changing its value, we can replace $H$ with
a smaller hypothesis set where only one hypothesis remains for each
unique binary vector $(h(x_1), h(x_2), \ldots, h(x_{2m}))$. The number
of such hypotheses is $\S_H(x_1^{2m})$, thus, by the union bound, the
following holds:
\begin{equation*}
\Pr_{\ssigma} \left[\sup_{h \in H}  \frac{ \sum_{i = 1}^m
  \sigma_i  (h(x_{m + i}) - h(x_i)) }{
  \sqrt[\alpha]{\frac{1}{2} [\sum_{i = 1}^m
  (h(x_{m + i}) + h(x_i))]} } > \e \, \bigg| \, x_1^{2m} \right]
\leq \S_H(x_1^{2m}) \exp \bigg(   \frac{-m^{\frac{2 (\alpha -
      1)}{\alpha}} \e^2 }{2^{\frac{\alpha + 2}{\alpha}} }  \bigg).
\end{equation*}
The result follows by taking expectations with respect to $x_1^{2m}$
and applying Lemma~\ref{lemma:lemma1} and
Lemma~\ref{lemma:lemma2} respectively.
\end{proof}

\begin{corollary}
  Let $1 < \alpha \leq 2$ and let $H$ be a hypothesis set of functions
  mapping $\cX$ to $\set{0, 1}$. Then, for any $\delta > 0$, each of
  the following two inequalities holds with probability at least $1 -
  \delta$:
\begin{align*}
& R(h) - \h R_S(h) \leq 
2^{\frac{\alpha +
    2}{2 \alpha}} \sqrt[\alpha]{R(h)}
\sqrt{\frac{\log \E[\S_H(x_1^{2m})] + \log\frac{4}{\delta}}{m^{\frac{2 (\alpha -
      1)}{\alpha}}} }\\
& \h R(h) - R_S(h)  \leq 2^{\frac{\alpha +
    2}{2 \alpha}} \sqrt[\alpha]{\h R(h)}
\sqrt{\frac{\log \E[\S_H(x_1^{2m})] + \log\frac{4}{\delta}}{m^{\frac{2 (\alpha -
      1)}{\alpha}}} }.
\end{align*}

\end{corollary}
\begin{proof}
  The result follows directly from Theorem~\ref{th:relative} by setting
  $\delta$ to match the upper bounds and taking the limit $\tau \to 0$.
\end{proof}
For $\alpha = 2$, the inequalities become
\begin{align}
\label{eq:second-1}
& R(h) - \h R_S(h) \leq 2
\sqrt{R(h) \frac{\log \E[\S_H(x_1^{2m})] + \log\frac{4}{\delta}}{m} }\\
\label{eq:second-2}
& \h R_S(h) - R(h)  \leq 2 
\sqrt{\h R(h) \frac{\log \E[\S_H(x_1^{2m})] + \log\frac{4}{\delta}}{m} },
\end{align}
with the familiar dependency
$O\left(\sqrt{\frac{\log(m/d)}{m/d}}\right)$. The advantage of these
relative deviations is clear. For small values of $R(h)$ (or $\h
R(h)$) these inequalities provide tighter guarantees than standard
generalization bounds. Solving the corresponding second-degree
inequalities in $\sqrt{R(h)}$ or $\sqrt{\h R(h)}$ leads to the
following results.
\begin{corollary}
  Let $1 < \alpha \leq 2$ and let $H$ be a hypothesis set of functions
  mapping $\cX$ to $\set{0, 1}$. Then, for any $\delta > 0$, each of
  the following two inequalities holds with probability at least $1 -
  \delta$:
\begin{align*}
& R(h) \leq \h R_S(h) + 2
\sqrt{\h R_S(h) \frac{\log \E[\S_H(x_1^{2m})] + \log\frac{4}{\delta}}{m} }
+ 4 \frac{\log \E[\S_H(x_1^{2m})] + \log\frac{4}{\delta}}{m}\\
& \h R_S(h) \leq R(h) + 2 
\sqrt{R(h) \frac{\log \E[\S_H(x_1^{2m})] + \log\frac{4}{\delta}}{m}
} + 4 \frac{\log \E[\S_H(x_1^{2m})] + \log\frac{4}{\delta}}{m}.
\end{align*}
\end{corollary}
\begin{proof}
The second-degree inequality corresponding to \eqref{eq:second-1} can
be written as
\begin{equation*}
\sqrt{R(h)}^2 - 2 \sqrt{R(h)} u - \h R_S(h) \leq 0,
\end{equation*}
with $u = \sqrt{\frac{\log \E[\S_H(x_1^{2m})] +
    \log\frac{4}{\delta}}{m} }$, and implies
$\sqrt{R(h)} \leq u + \sqrt{u^2 + \h R_S(h)}$. Squaring both sides
gives:
\begin{align*}
R(h) 
\leq \Big[ u + \sqrt{u^2 + \h R_S(h)} \Big]^2
& = u^2 + 2 u \sqrt{u^2 + \h R_S(h)} + u^2 + \h R_S(h)\\
& \leq u^2 + 2 u \Big(\sqrt{u^2} + \sqrt{\h R_S(h)}\Big) + u^2 + \h
R_S(h)\\
& = 4 u^2 + 2 u \sqrt{\h R_S(h)} + \h
R_S(h).
\end{align*}
The second inequality can be proven in the same way from \eqref{eq:second-2}.
\end{proof}
The learning bounds of the corollary make clear the presence of two
terms: a term in $O(1/m)$ and a term in $O(1/\sqrt{m})$ which admits
as a factor $\h R_S(h)$ or $R(h)$ and which for small values of these
terms can be more favorable than standard bounds.
Theorem~\ref{th:relative} can also be used to prove the following
relative deviation bounds.

The following theorem and its proof assuming the result of
Theorem~\ref{th:relative} were given by \cite{AnthonyBartlett99}.

\begin{theorem}
\label{th:rel}
For all $0 < \e < 1$, $\nu > 0$, the following inequalities hold:
\begin{align*}
& \Pr_{S \sim D^m} \bigg[\sup_{h \in H} \frac{R(h) - \h R_S(h)}{R(h) +
  \h R(h) + \nu} > 
\e \bigg]
\leq 4 \E[\S_H(x_1^{2m})] \exp \bigg( \frac{-m \nu \e^2}{2 (1 - \e^2)}  \bigg)\\
& \Pr_{S \sim D^m} \bigg[\sup_{h \in H} \frac{\h R_S(h) - R(h)}{R(h) +
  \h R(h) + \nu} > 
\e \bigg]
\leq 4 \E[\S_H(x_1^{2m})] \exp \bigg( \frac{-m \nu \e^2}{2 (1 - \e^2)}  \bigg).
\end{align*}
\end{theorem}

\begin{proof}
  We prove the first statement, the proof of the second statement is
  identical modulo the permutation of the roles of $R(h)$ and $\h
  R_S(h)$. To do so, it suffices to determine $\e' > 0$
  such that
\begin{equation*}
\Pr_{S \sim D^m} \bigg[\sup_{h \in H} \frac{R(h) - \h R_S(h)}{R(h) + \h R(h) + \nu} > 
\e \bigg]
\le
\Pr_{S \sim D^m} \bigg[\sup_{h \in H} \frac{R(h) - \h R_S(h)}{\sqrt[\alpha]{R(h) + \tau}} > \e' \bigg],
\end{equation*}
since we can then apply theorem~\ref{th:relative} with $\alpha = 2$ to
bound the right-hand side and take the limit as $\tau \to 0$ to
eliminate the $\tau$-dependence. To find such a choice of $\e'$, we
begin by observing that for any $h \in H$,
\begin{equation}
\label{eq:4}
\frac{R(h) - \h R_S(h)}{R(h) + \h R(h) + \nu} \le \e \Leftrightarrow
R(h)  \le  \frac{1+\e}{1-\e} \h R_S(h)  +  \frac{\e}{1-\e} \nu.
\end{equation}
Assume now that $\frac{R(h) - \h R_{S}(h)}{\sqrt{R(h) +\tau}} \le \e'$
for some $\e' > 0$, which is equivalent to $R(h) \le \h R_{S}(h) + \e'
\sqrt{R(h) + \tau}$. We will prove that this implies \eqref{eq:4}.  To
show that, we distinguish two cases, $R(h) + \tau \le \mu^2 \e'^2$ and
$R(h) + \tau > \mu^2 \e'^2$, with $\mu > 1$.
The first case implies the following:
\begin{equation*}
R(h) + \tau \le \mu^2 \e'^2
\Rightarrow R(h) \le \h R_{S}(h) + \e' \sqrt{\mu^2 \e'^2}
\Leftrightarrow  R(h) \le \h R_{S}(h) + \mu \e'^2.
\end{equation*}
The second case $R(h) + \tau > \mu^2 \e'^2$ is equivalent
to $\e' < \frac{\sqrt{R(h) + \tau}}{\mu}$ and implies
\begin{align*}
\e' < \frac{\sqrt{R(h) + \tau}}{\mu}
\Rightarrow  R(h) \le \h R_S(h) + \frac{R(h) + \tau}{\mu}
\Leftrightarrow R(h) \le \frac{\mu}{\mu-1} \h R_S(h) + \frac{\tau}{\mu-1}.
\end{align*}
Observe now that since $\frac{\mu}{\mu-1} > 1$, both cases imply
\begin{equation}
\label{eq:hither}
R(h) \le \frac{\mu}{\mu-1} \h R_S(h) + \frac{\tau}{\mu-1} + \mu \e'^2.
\end{equation}
We now choose $\e'$ and $\mu$ to make \eqref{eq:hither} match
\eqref{eq:4} by setting $\frac{\mu}{\mu-1} = \frac{1+\e}{1-\e}$ and
$\frac{\tau}{\mu-1} + \mu \e'^2 = \frac{\e}{1-\e} \nu$, which gives:
\begin{equation*}
\mu = \frac{1+\e}{2 \e} \hspace{30pt}  \e'^2 =  \frac{2 \e^2 (\nu - 2 \tau)}{1 - \e^2}.
\end{equation*}
With these choices, the following inequality holds for all $h \in H$:
\begin{align*}
  \frac{R(h) - \h R_S(h)}{\sqrt{R(h) + \tau}} \le \e' \Rightarrow
  \frac{R(h) - \h R_S(h)}{R(h) + \h R(h) + \nu} \le \e,
\end{align*}
which concludes the proof.
\end{proof}
The following corollary was given by \cite{AnthonyBartlett99}.

\begin{corollary}
\label{cor:6}
For all $\e > 0$, $v > 0$, the following inequality holds:
\begin{equation*}
\Pr_{S \sim D^m} \bigg[\sup_{h \in H} R(h) - (1 + v) \h R_S(h) > \e 
\bigg]
\leq 4 \E[\S_H(x_1^{2m})] \exp \bigg( \frac{-m v \e }{4 (1 + v)}  \bigg).
\end{equation*}
\end{corollary}
\begin{proof}
Observe that 
\begin{equation*}
\frac{R(h) - \h R_S(h)}{R(h) + \h R(h) + \nu} > \e 
\Leftrightarrow R(h) - \h R_S(h) > (R(h) + \h R(h) + \nu) \e
\Leftrightarrow R(h) > \frac{1 + \e}{1 - \e} \h R(h) + 
\frac{\e \nu}{1 - \e}.
\end{equation*}
To derive the statement of the corollary from that of
Theorem~\ref{th:rel}, we identify $\frac{1 + \e}{1 - \e}$ with $1 +
v$, which gives $\e (2 + v) = v$, that is we choose $\e = \frac{v}{2 +
  v}$, and similarly identify $\frac{\e \nu}{1 - \e}$ with $\e'$, that
is $\e' = \frac{\frac{v}{2 + v}}{\frac{2}{2 + v}}\nu = \frac{v}{2}
\nu$, thus we choose $\nu = \frac{2}{v} \e'$. With these choices of
$\e'$ and $\nu$, the coefficient in the exponential appearing in the
bounds of Theorem~\ref{th:rel} can be rewritten as follows:
$\frac{v \e^2}{2(1 - \e^2)} = \frac{2 \e'}{2 v} \frac{\frac{v^2}{(2 +
    v)^2}}{\frac{4v + 4}{(2 + v)^2}} = \frac{\e'}{v} \frac{v^2}{4 (v
  + 1)} = \frac{\e' v}{4 (v + 1)}$, which concludes the proof.
\end{proof}
The result of Corollary~\ref{cor:6} is remarkable since it shows that
a fast convergence rate of $O(1/m)$ can be achieved provided that we
settle for a slightly larger value than the empirical error, one
differing by a fixed factor $(1 + v)$.  The following is an immediate
corollary when $\h R_S(h) = 0$, where we take $v \to \infty$. 

\begin{corollary}
\label{cor:7}
For all $\e > 0$, $v > 0$, the following inequality holds:
\begin{equation*}
\Pr_{S \sim D^m} \bigg[\exists h \in H\colon R(h) > \e \wedge 
 \h R_S(h) = 0 \bigg]
\leq 4 \E[\S_H(x_1^{2m})] \exp \bigg( \frac{-m \e }{4}  \bigg).
\end{equation*}
\end{corollary}
This is the familiar fast rate convergence result for separable cases.

\ignore{
\begin{theorem}
For all $\e > 0$, $\nu > 0$, the following inequality holds:
\begin{equation*}
\Pr_{S \sim D^m} \bigg[\exists h \in H\colon R(h) > \e \wedge 
 \h R_S(h) \leq (1 - \nu) R(h) \bigg] \leq 4 \E[\S_H(x_1^{2m})] \exp \bigg( \frac{-m \nu^2 \e}{4}  \bigg).
\end{equation*}
In particular, for $\nu = 1$, the following holds:
\begin{equation*}
\Pr_{S \sim D^m} \bigg[\exists h \in H\colon R(h) > \e \wedge 
 \h R_S(h) = 0 \bigg] \leq 4 \E[\S_H(x_1^{2m})] \exp \bigg( \frac{-m \e}{4}  \bigg).
\end{equation*}
\end{theorem}
}

\section{Generalization bounds for unbounded losses}
\label{sec:unbounded}

In this section we will make use of the relative deviation bounds
given in the previous section to prove generalization bounds for
unbounded loss functions under the assumption that the moment of order
$\alpha$ of the loss is bounded. We will start with the case $1 <
\alpha \leq 2$ and then move on to considering the case when $\alpha >
2$. As already indicated earlier, the one-sided version of the results
presented in this section were given by \cite{vapnik98} with slightly
different constants, but the proofs do not seem to be correct or
complete. The second statements in all these results (other side of
the inequality) are new. Our proofs for both sets of results are new.

\subsection{Bounded moment with $1 < \alpha \leq 2$}
\label{sec:unbounded1}

Our first theorem reduces the problem of deriving a relative deviation
bound for an unbounded loss function with $\L_\alpha(h) = \E_{z \sim
  D}[L(h, z)^\alpha]< +\infty$ for all $h \in H$, to that of relative
deviation bound for binary classification. To simplify the
presentation of the results, in what follows we will use the shorthand
$\Pr [L(h, z) > t]$ instead of $\Pr_{z \sim D}[L(h, z) > t]$, and
similarly $\h \Pr [L(h, z) > t]$ instead of $\Pr_{z \sim \h D}[L(h, z)
> t]$.

\begin{theorem}
\label{th:main1}
Let $1 < \alpha \leq 2$, $0 < \e \leq 1$, and $0 < \tau^{\frac{\alpha
    - 1}{\alpha}} < \e^{\frac{\alpha}{\alpha - 1}}$. For any loss
function $L$ (not necessarily bounded) and hypothesis set $H$ such
that $\L_\alpha(h) < +\infty$ for all $h \in H$,
the following two inequalities hold:
\begin{align*}
  & \Pr \bigg[\sup_{h \in H} \frac{\L(h) - \h
    \L_S(h)}{\sqrt[\alpha]{\L_\alpha(h) + \tau}} > \Gamma(\alpha,
  \e)\, \e \bigg] \leq \Pr \bigg[\sup_{h \in H, t \in \Rset} \frac{
    \Pr[L(h, z) > t] - \h \Pr[L(h, z) > t]
  }{\sqrt[\alpha]{\Pr[L(h, z) > t] + \tau}} > \e \bigg]\\
  & \Pr \bigg[\sup_{h \in H} \frac{ \L(h) - \h
    \L_S(h)}{\sqrt[\alpha]{\L_\alpha(h) + \tau}}
  > \Gamma(\alpha, \e)\, \e \bigg] \leq \Pr \bigg[\sup_{h \in H, t \in
    \Rset} \frac{ \h \Pr[L(h, z) > t] - \Pr[L(h, z) > t]
  }{\sqrt[\alpha]{ \h \Pr[L(h, z) > t] + \tau}} > \e \bigg],
\end{align*}
with $\Gamma(\alpha, \e) = \frac{\alpha - 1}{\alpha} (1 +
\tau)^{\frac{1}{\alpha}} + \frac{1}{\alpha} \big( \frac{\alpha}{\alpha
  - 1} \big)^{\alpha - 1} (1 + \big(\frac{\alpha -
  1}{\alpha}\big)^\alpha \tau^{\frac{1}{\alpha}})^{\frac{1}{\alpha}}
\Big[ 1 + \frac{\log (1/\e)}{\big( \frac{\alpha}{\alpha - 1}
  \big)^{\alpha - 1}} \Big]^{\frac{\alpha - 1}{\alpha}}$ \ignore{
  $\Gamma(\beta, \e) = \frac{1}{\beta} + \big( \frac{\beta - 1}{\beta}
  \big) \beta^{\frac{1}{\beta - 1}} \Big [1 +
  \frac{\log(1/\e)}{\beta^{\frac{1}{\beta - 1}}}
  \Big]^{\frac{1}{\beta}}$, with $\frac{1}{\alpha} + \frac{1}{\beta} =
  1$}.
\end{theorem}

\begin{proof}
  We prove the first statement. The second statement can be shown in a
  very similar way. 
  Fix $ 1 < \alpha \le 2$ and $\e > 0$ and§
  assume that for any $h \in H$ and $t \geq 0$, the following
  holds:
\begin{equation}
\label{eq:14}
\frac{\Pr[L(h, z) > t] - \h \Pr[L(h, z) > t] }{\sqrt[\alpha]{\Pr[L(h, z) > t] + \tau}}  \leq \e.
\end{equation}
We show that this implies that for any $h \in H$, $\frac{\L(h) - \h
  \L_{S}(h)}{\sqrt[\alpha]{\L_\alpha(h) + \tau}} \leq \Gamma(\alpha, \e)  \e $. By the properties of the
Lebesgue integral, we can write
\begin{align*}
&  \L(h) = \E_{z \sim D}[L(h, z)] = \int_{0}^{+\infty} \Pr[L(h, z) > t]\,
  dt \\
& \h \L(h) = \E_{z \sim \h D}[L(h, z)] = \int_{0}^{+\infty} \h \Pr[L(h, z) > t]\, dt,
\end{align*}
and, similarly,
\begin{equation*}
  \L_\alpha(h) = \L_\alpha(h) = \int_{0}^{+\infty} \Pr[L^\alpha(h, z) > t]\, dt = \int_{0}^{+\infty} \alpha t^{\alpha - 1} \Pr[L(h, z) > t]\, dt.
\end{equation*}
In what follows, we use the notation $I_\alpha =
\L_\alpha(h) + \tau$.  Let $t_0 = s I_\alpha^{\frac{1}{\alpha}}$ and
$t_1 = t_0 \left[ \frac{1}{\e} \right]^{\frac{1}{\alpha - 1}}$ for
$s > 0$. To bound $\L(h) - \h \L(h)$, we simply bound $\Pr[L(h, z) > t] - \h \Pr[L(h, z) > t]$ by $\Pr[L(h, z) > t]$ for large values of $t$,
that is $t >  t_1$, and use inequality (\ref{eq:14}) for smaller
values of $t$:
\begin{align*}
\L(h) - \h \L(h)
&= \int_{0}^{+\infty} \Pr[L(h, z) > t] - \h \Pr[L(h, z) > t]\, dt\\
& \leq \int_{0}^{t_1}  \e \sqrt[\alpha]{\Pr[L(h, z) > t] + \tau} \,  dt +
\int_{t_1}^{+\infty} \Pr[L(h, z) > t] \, dt.
\end{align*}
For relatively small values of $t$, $\Pr[L(h, z) > t]$ is close to
one. Thus, we can write
\begin{align*}
\L(h) - \h \L(h)
&\leq \int_{0}^{t_0} \e \sqrt[\alpha]{1 + \tau}\,  dt +
\int_{t_0}^{t_1} \e \sqrt[\alpha]{\Pr[L(h, z) > t] + \tau} \, dt + \int_{t_1}^{+\infty} \Pr[L(h, z) > t] dt
\\
&= \int_{0}^{+\infty} f(t) g(t)\,  dt,
\end{align*}
with
\begin{equation*}
f(t) =
\begin{cases}
\gamma_1 I_\alpha^{\frac{\alpha - 1}{\alpha^2}} \e \sqrt[\alpha]{1 + \tau} & \text{if } 0 \leq t \leq t_0\\
\gamma_2 \left[ \alpha t^{\alpha - 1} (\Pr[L(h, z) > t] + \tau) \right]^{\frac{1}{\alpha}}\, \e & \text{if } t_0< t \leq t_1\\
\gamma_2 \left[ \alpha t^{\alpha - 1} \Pr[L(h, z) > t] \right]^{\frac{1}{\alpha}}\, \e & \text{if } t_1 < t.
\end{cases}
\quad
g(t) =
\begin{cases}
\frac{1}{\gamma_1 I_\alpha^{\frac{\alpha - 1}{\alpha^2}}} & \text{if } 0 \leq t \leq t_0\\
\frac{1}{\gamma_2 (\alpha t^{\alpha - 1})^{\frac{1}{\alpha}}} & \text{if } t_0 < t \leq t_1\\
\frac{\Pr[L(h, z) > t]^{\frac{\alpha - 1}{\alpha}}}{\gamma_2 (\alpha t^{\alpha - 1})^{\frac{1}{\alpha}}} \frac{1}{\e} & \text{if } t_1 < t,
\end{cases}
\end{equation*}
where $\gamma_1, \gamma_2$ are positive parameters that we shall select later.
Now, since $\alpha > 1$, by H\"{o}lder's inequality,
\begin{align*}
\L(h) - \h \L(h)
& \leq \left[ \int_{0}^{+\infty} f(t)^\alpha\, dt \right]^{\frac{1}{\alpha}} \, \left[ \int_{0}^{+\infty} g(t)^{\frac{\alpha}{\alpha - 1}}\,  dt \right]^{\frac{\alpha - 1}{\alpha}}.
\end{align*}
The first integral on the right-hand side can be bounded as follows:
\begin{align*}
\int_{0}^{+\infty}  f(t)^\alpha\, dt
& = \int_{0}^{t_0}  (1 + \tau) (\gamma_1 I_\alpha^{\frac{\alpha - 1}{\alpha^2}}
\e)^\alpha  \, dt + 
\gamma_2^\alpha \e^\alpha \tau \int_{t_0}^{t_1} \alpha t^{\alpha - 1} dt 
+
\gamma_2^\alpha \int_{t_0}^{+\infty}  \alpha t^{\alpha - 1} \Pr[L(h, z) > t]  \e^\alpha \, dt\\
& \leq (1 + \tau) \gamma_1^\alpha I_\alpha^{\frac{\alpha - 1}{\alpha}} t_0
\e^\alpha + 
\gamma_2^\alpha \e^\alpha \tau (t_1^\alpha - t_0^\alpha)+ 
\gamma_2^\alpha \e^\alpha I_\alpha \\
& \leq (\gamma_1^\alpha (1 + \tau) s + 
\gamma_2^\alpha (1 + s^\alpha
(1/\e)^{\frac{\alpha}{\alpha - 1}} \tau) ) \e^\alpha I_\alpha\\
& \leq (\gamma_1^\alpha (1 + \tau) s + \gamma_2^\alpha (1 + s^\alpha \tau^{\frac{1}{\alpha}})) \e^\alpha I_\alpha.
\end{align*}
Since $t_1/t_0 = (1/\e)^{\frac{1}{\alpha - 1}}$, the second one can be computed and bounded following
\begin{align*}
\int_{0}^{+\infty}  g(t)^{\frac{\alpha}{\alpha - 1}}\, dt
& = \int_{0}^{t_0}  \frac{dt}{\gamma_1^{\frac{\alpha}{\alpha - 1}} I_\alpha^{\frac{1}{\alpha}}} + \int_{t_0}^{t_1} \frac{1}{\gamma_2^{\frac{\alpha}{\alpha - 1}} \alpha^{\frac{1}{\alpha - 1}} } \frac{dt}{t} +
\int_{t_1}^{+\infty} \frac{\Pr[L(h, z) > t]}{\gamma_2^{\frac{\alpha}{\alpha - 1}} \alpha^{\frac{1}{\alpha - 1}} \e^{\frac{\alpha}{\alpha - 1}} t} dt\\
& = \frac{s}{\gamma_1^{\frac{\alpha}{\alpha - 1}}} + \frac{1}{\gamma_2^{\frac{\alpha}{\alpha - 1}} (\alpha - 1) \alpha^{\frac{1}{\alpha - 1}} } \log \frac{1}{\e} + \int_{t_1}^{+\infty} \frac{\alpha t^{\alpha - 1} \Pr[L(h, z) > t]}{\gamma_2^{\frac{\alpha}{\alpha - 1}}  (\alpha \e)^{\frac{\alpha}{\alpha - 1}} t^\alpha } dt\\
& \leq \frac{s}{\gamma_1^{\frac{\alpha}{\alpha - 1}}} + \frac{1}{\gamma_2^{\frac{\alpha}{\alpha - 1}} (\alpha - 1) \alpha^{\frac{1}{\alpha - 1}} } \log \frac{1}{\e} + \int_{t_1}^{+\infty} \frac{\alpha t^{\alpha - 1} \Pr[L(h, z) > t]}{\gamma_2^{\frac{\alpha}{\alpha - 1}}  (\alpha \e)^{\frac{\alpha}{\alpha - 1}} t_1^\alpha} dt\\
& \leq \frac{s}{\gamma_1^{\frac{\alpha}{\alpha - 1}}} + \frac{1}{\gamma_2^{\frac{\alpha}{\alpha - 1}} (\alpha - 1) \alpha^{\frac{1}{\alpha - 1}} } \log \frac{1}{\e} + \frac{I_\alpha}{\gamma_2^{\frac{\alpha}{\alpha - 1}} (\alpha \e)^{\frac{\alpha}{\alpha - 1}} s^\alpha I_\alpha (\frac{1}{\e})^{\frac{\alpha}{\alpha - 1}}} \\
& = \frac{s}{\gamma_1^{\frac{\alpha}{\alpha - 1}}} + \frac{1}{\gamma_2^{\frac{\alpha}{\alpha - 1}}} \left( \frac{1}{(\alpha - 1) \alpha^{\frac{1}{\alpha - 1}} }  \log \frac{1}{\e} + \frac{1}{\alpha^{\frac{\alpha}{\alpha - 1}} s^\alpha }  \right).
\end{align*}
Combining the bounds obtained for these integrals yields directly
\begin{align*}
&\L(h) - \h \L(h) \\
& \leq \left[ (\gamma_1^\alpha (1 + \tau) s + \gamma_2^\alpha (1 +
  s^\alpha \tau^{\frac{1}{\alpha}})) \e^\alpha I_\alpha \right]^{\frac{1}{\alpha}} \left[ \frac{s}{\gamma_1^{\frac{\alpha}{\alpha - 1}}} + \frac{1}{\gamma_2^{\frac{\alpha}{\alpha - 1}}} \left( \frac{1}{(\alpha - 1) \alpha^{\frac{1}{\alpha - 1}} }  \log \frac{1}{\e} + \frac{1}{\alpha^{\frac{\alpha}{\alpha - 1}} s^\alpha }  \right) \right]^{\frac{\alpha - 1}{\alpha}} \\
& = (\gamma_1^\alpha (1 + \tau) s + \gamma_2^\alpha (1 + s^\alpha \tau^{\frac{1}{\alpha}}))^{\frac{1}{\alpha}} 
\left[ \frac{s}{\gamma_1^{\frac{\alpha}{\alpha - 1}}} + \frac{1}{\gamma_2^{\frac{\alpha}{\alpha - 1}}} \left( \frac{1}{(\alpha - 1) \alpha^{\frac{1}{\alpha - 1}} }  \log \frac{1}{\e} + \frac{1}{\alpha^{\frac{\alpha}{\alpha - 1}} s^\alpha }  \right) \right]^{\frac{\alpha - 1}{\alpha}}
\e I_\alpha^{\frac{1}{\alpha}}.
\end{align*}
Observe that the expression on the right-hand side can be rewritten as
$\| \u \|_\alpha \| \v \|_{\frac{\alpha}{\alpha - 1}} \ \e
I_\alpha^{\frac{1}{\alpha}}$ where the vectors $\u$ and $\v$ are
defined by $\u = (\gamma_1 (1 + \tau)^{\frac{1}{\alpha}}
s^{\frac{1}{\alpha}}, \gamma_2 (1 + s^\alpha
\tau^{\frac{1}{\alpha}})^{\frac{1}{\alpha}})$ and $\v = (v_1, v_2) =
\bigg(\frac{s^{\frac{\alpha - 1}{\alpha}}}{\gamma_1},
\frac{1}{\gamma_2} \Big[ \frac{1}{(\alpha - 1) \alpha^{\frac{1}{\alpha
      - 1}} } \log \frac{1}{\e} +
\frac{1}{\alpha^{\frac{\alpha}{\alpha - 1}} s^\alpha }
\Big]^{\frac{\alpha - 1}{\alpha}} \bigg)$. The inner product $\u \cdot
\v$ does not depend on $\gamma_1$ and $\gamma_2$ and by the properties
of H\"{o}lder's inequality can be reached when $\u$ and the vector
$\v' = (v_1^{\frac{1}{\alpha - 1}}, v_2^{\frac{1}{\alpha - 1}})$ are
collinear. $\gamma_1$ and $\gamma_2$ can be chosen so that
$\text{det}(\u, \v') = 0$, since this condition can be rewritten as
\begin{equation}
 s^{\frac{1}{\alpha}} (1 + \tau)^{\frac{1}{\alpha}} \frac{\gamma_1}{\gamma_2^{\frac{1}{\alpha - 1}}} \Big[ \frac{1}{(\alpha - 1) \alpha^{\frac{1}{\alpha
      - 1}} } \log \frac{1}{\e} +
\frac{1}{\alpha^{\frac{\alpha}{\alpha - 1}} s^\alpha }
\Big]^{\frac{1}{\alpha}} - s^{\frac{1}{\alpha}} (1 + s^\alpha \tau^{\frac{1}{\alpha}})^{\frac{1}{\alpha}} \frac{\gamma_2}{\gamma_1^{\frac{1}{\alpha - 1}}} = 0,
\end{equation}
or equivalently,
\begin{equation}
 \left( \frac{\gamma_1}{\gamma_2} \right)^{\frac{\alpha}{\alpha - 1}} \Big[ \frac{1}{(\alpha - 1) \alpha^{\frac{1}{\alpha
      - 1}} } \log \frac{1}{\e} +
\frac{1}{\alpha^{\frac{\alpha}{\alpha - 1}} s^\alpha }
\Big]^{\frac{1}{\alpha}} - (1 + s^\alpha \tau^{\frac{1}{\alpha}})^{\frac{1}{\alpha}}  = 0.
\end{equation}
Thus, for such values of $\gamma_1$ and $\gamma_2$, the following
inequality holds:
\begin{align*}
\L(h) - \h \L(h)
& \leq (\u \cdot \v)\, \e I_\alpha^{\frac{1}{\alpha}} = f(s) \, \e I_\alpha^{\frac{1}{\alpha}},
\end{align*}
with 
\begin{align*}
f(s) 
& = (1 + \tau)^{\frac{1}{\alpha}}
s + (1 + s^\alpha \tau^{\frac{1}{\alpha}})^{\frac{1}{\alpha}} \Big[ \frac{1}{(\alpha - 1) \alpha^{\frac{1}{\alpha
      - 1}} } \log \frac{1}{\e} +
\frac{1}{\alpha^{\frac{\alpha}{\alpha - 1}} s^\alpha }
\Big]^{\frac{\alpha - 1}{\alpha}} \\
& = (1 + \tau)^{\frac{1}{\alpha}} s + \frac{(1 + s^\alpha \tau^{\frac{1}{\alpha}})^{\frac{1}{\alpha}}}{\alpha} \Big[ \frac{\alpha}{(\alpha - 1) } \log \frac{1}{\e} +
\frac{1}{s^\alpha }
\Big]^{\frac{\alpha - 1}{\alpha}}.
\end{align*}
Setting $s = \frac{\alpha - 1}{\alpha}$ yields the statement of the theorem.
\end{proof}
The next corollary follows immediately by upper bounding the
right-hand side of the learning bounds of theorem~\ref{th:main1} using
theorem~\ref{th:relative}. It provides learning bounds for unbounded
loss functions in terms of the growth functions in the case $1 <
\alpha \leq 2$.

\begin{corollary}
\label{cor:main1vc}
Let $\e < 1$, $1 < \alpha \le 2$, and $0 < \tau^{\frac{\alpha -
    1}{\alpha}} < \e^{\frac{\alpha}{\alpha - 1}}$. For any loss
function $L$ (not necessarily bounded) and hypothesis set $H$ such
that $\L_\alpha(h) < +\infty$ for all $h \in H$, the following
inequalities hold:
\begin{align*}
& \Pr \bigg[\sup_{h \in H} \frac{\L(h) - \h \L(h)}{\sqrt[\alpha]{\L_\alpha(h) + \tau}} > 
\Gamma(\alpha, \epsilon) \epsilon \bigg ]
\leq 4 \, \E[\S_Q(z_1^{2m})] \exp \bigg(   \frac{-m^{\frac{2 (\alpha -
      1)}{\alpha}} \e^2 }{2^{\frac{\alpha + 2}{\alpha}} }  \bigg)\\
& \Pr \bigg[\sup_{h \in H} \frac{\h \L(h) - \L(h)}{\sqrt[\alpha]{\h \L_\alpha(h) + \tau}} > 
\Gamma(\alpha, \epsilon) \epsilon   \bigg]
\leq 4 \, \E[\S_Q(z_1^{2m})] \exp \bigg(   \frac{-m^{\frac{2 (\alpha -
      1)}{\alpha}} \e^2 }{2^{\frac{\alpha + 2}{\alpha}} }  \bigg),
\end{align*}
where $Q$ is the set of functions $Q = \set{z \mapsto 1_{L(h, z) > t}
\mid h \in H, t \in \mathbb{R}}$, and
 $\Gamma(\alpha, \e) = 
\frac{\alpha - 1}{\alpha} (1 + \tau)^{\frac{1}{\alpha}} +
\frac{1}{\alpha} \big( \frac{\alpha}{\alpha - 1} \big)^{\alpha - 1} (1
+ \big(\frac{\alpha - 1}{\alpha}\big)^\alpha \tau^{\frac{1}{\alpha}})^{\frac{1}{\alpha}} \Big[ 1 +  \frac{\log (1/\e)}{\big( \frac{\alpha}{\alpha - 1} \big)^{\alpha - 1}}  \Big]^{\frac{\alpha - 1}{\alpha}}$.
\end{corollary}
The following corollary gives the explicit result for $\alpha = 2$.
\begin{corollary}
\label{cor:main1vcAlpha2}
Let $\e < 1$ and $0 < \tau < \e^4$. For any loss
function $L$ (not necessarily bounded) and hypothesis set $H$ such
that $\L_2(h) < +\infty$ for all $h \in H$, the following inequalities
hold:
\begin{align*}
& \Pr \bigg[\sup_{h \in H} \frac{\L(h) - \h \L(h)}{\sqrt{\L_2(h) + \tau}} > 
\Gamma(2, \e) \e   \bigg]
\leq 
4 \, \E[\S_Q(z_1^{2m})] \exp \bigg(   \frac{-m \e^2 }{4}  \bigg)\\
& \Pr \bigg[\sup_{h \in H} \frac{\h \L(h) - \L(h)}{\sqrt{\h \L_2(h) + \tau}} > 
\Gamma(2, \e) \e   \bigg]
\leq 4 \, \E[\S_Q(z_1^{2m})] \exp \bigg(   \frac{-m \e^2 }{4}  \bigg),
\end{align*}
with $\Gamma(2, \e) = \big( \frac{\sqrt{1+\tau}}{2} + \sqrt{1 +
  \frac{1}{4} \sqrt{\tau}} \sqrt{1 + \frac{1}{2} \log \frac{1}{\e}}
\big) $ and $Q$ the set of functions $Q = \set{z \mapsto 1_{L(h, z) >
    t} \mid h \in H, t \in \mathbb{R}}$.
\end{corollary}

\begin{corollary}
  Let $L$ be a loss function (not necessarily bounded) and $H$ a
  hypothesis set such that $\L_2(h) < +\infty$ for all $h \in
  H$. Then, for any $\delta > 0$, with probability at least $1 -
  \delta$, each of the following inequalities holds for all $h \in H$:
\begin{align*}
  & \L(h) \leq \h \L_S(h) + 2 \sqrt{\L_2(h)}
    \sqrt{\vphantom{\big()} \smash[b]{   \frac{2 \log \E[\S_Q(z_1^{2m})] +
    \log{\frac{1}{\delta}}  }{m}}}  \Gamma_0\Bigg(2, 2\sqrt{\vphantom{\big()} \smash[b]{   \frac{2 \log \E[\S_Q(z_1^{2m})] + \log{\frac{1}{\delta}}  }{m}}}\Bigg)\\
  & \h \L_S(h) \leq \L(h) + 2 \sqrt{\mathstrut \smash[b]{\h \L_2(h)}}
    \sqrt{\vphantom{\big()} \smash[b]{   \frac{2 \log \E[\S_Q(z_1^{2m})] +
    \log{\frac{1}{\delta}}  }{m}}}  \Gamma_0\Bigg(2, 2\sqrt{\vphantom{\big()} \smash[b]{   \frac{2 \log \E[\S_Q(z_1^{2m})] + \log{\frac{1}{\delta}}  }{m}}}\Bigg),
\end{align*}
where $Q$ is the set of functions $Q = \set{z \mapsto 1_{L(h, z) > t}
  \mid h \in H, t \in \mathbb{R}}$ and
$\Gamma_0(2, \e) = \frac{1}{2} + \sqrt{\mathstrut \smash[b]{1 + \frac{1}{2} \log \frac{1}{\e}}}$.
\end{corollary}

\begin{proof}
For any $\e > 0$, let $f(\e) = \Gamma_0(2, \e) \e$. Then, by Corollary~\ref{cor:main1vcAlpha2},
\begin{align*}
& \Pr \bigg[\sup_{h \in H} \frac{\L(h) - \h \L(h)}{\sqrt{\L_2(h) + \tau}} > 
\e \bigg]
\leq 
4 \, \E[\S_Q(z_1^{2m})] \exp \bigg(   \frac{-m [f^{-1}(\e)]^2 }{4}  \bigg).
\end{align*}
Setting the right-hand side to $\e$ and using inversion yields
immediately the first inequality. The second inequality is proven 
in the same way.
\end{proof}
Observe that, modulo the factors in $\Gamma_0$, the bounds of the
corollary admit the standard $(1/\sqrt{m})$ dependency and that the
factors in $\Gamma_0$ are only logarithmic in $m$.

\subsection{Bounded moment with $\alpha > 2$}
\label{sec:unbounded2}

This section gives two-sided generalization bounds for unbounded
losses with finite moments of order $\alpha$, with $\alpha > 2$. As
for the case $1 < \alpha < 2$, the one-sided version of our bounds
coincides with that of \cite{vapnik98,vapnik06} modulo a constant
factor, but, here again, the proofs given by Vapnik in both books seem
to be incorrect.

\begin{proposition}
\label{prop:1}
Let $\alpha > 2$. For any loss function $L$ (not
necessarily bounded) and hypothesis set $H$ such that $0 <
\L_\alpha(h) < +\infty$ for all $h \in  H$, the following two
inequalities hold:
\begin{equation*}
\int_0^{+\infty}  \sqrt{\Pr[L(h, z) > t]} dt \leq \Psi(\alpha) \sqrt[\alpha]{\L_\alpha(h)} \quad \text{and} \quad
\int_0^{+\infty}  \sqrt{\h \Pr[L(h, z) > t]} dt \leq \Psi(\alpha) \sqrt[\alpha]{\h \L_\alpha(h)},
\end{equation*}
where $\Psi(\alpha) = \big( \frac{1}{2} \big)^{\frac{2}{\alpha}}
\big( \frac{\alpha}{\alpha - 2} \big)^{\frac{\alpha - 1}{\alpha}}$.
\end{proposition}

\begin{proof}
  We prove the first inequality. The second can be proven in a very
  similar way. Fix $\alpha > 2$ and $h \in H$.  As
  in the proof of Theorem~\ref{th:main1}, we bound $\Pr[L(h, z) > t]$ by
  $1$ for $t$ close to $0$, say $t \leq t_0$ for some $t_0 >
  0$ that we shall later determine. We can write
\begin{equation*}
\int_0^{+\infty}  \sqrt{\Pr[L(h, z) > t]} dt \leq  
\int_{0}^{t_0}  1 dt + \int_{t_0}^{+\infty}  \sqrt{\Pr[L(h, z) > t]} dt
= \int_0^{+\infty}  f(t) g(t) dt,
\end{equation*}
with functions $f$ and $g$ defined as follows:
\begin{equation*}
f(t) =
\begin{cases}
\gamma I_\alpha^{\frac{\alpha - 1}{2 \alpha}} & \text{if } 0 \leq t \leq t_0\\
\alpha^{\frac{1}{2}} t^{\frac{\alpha - 1}{2}} \Pr[L(h, z) > t]^{\frac{1}{2}} & \text{if } t_0 < t.
\end{cases}
\quad
g(t) =
\begin{cases}
\frac{1}{\gamma I_\alpha^{\frac{\alpha - 1}{2 \alpha}}} & \text{if } 0 \leq t \leq t_0\\
\frac{1}{\alpha^{\frac{1}{2}} t^{\frac{\alpha - 1}{2}}} & \text{if } t_0 < t,
\end{cases}
\end{equation*}
where $I_\alpha = \L_\alpha(h)$ and where $\gamma$ is a positive
parameter that we shall select later.  By the Cauchy-Schwarz
inequality,
\begin{equation*}
\int_0^{+\infty}  \sqrt{\Pr[L(h, z) > t]} dt 
\leq \left( \int_0^{+\infty}  f(t)^2 dt \right)^{\frac{1}{2}} \left(\int_0^{+\infty}  g(t)^2 dt \right)^{\frac{1}{2}}.
\end{equation*}
Thus, we can write
\begin{multline*}
\int_0^{+\infty}  \sqrt{\Pr[L(h, z) > t]} dt \\
\begin{aligned}
& \leq \left( \gamma^2 I_\alpha^{\frac{\alpha - 1}{\alpha}} t_0 + \int_{t_0}^{+\infty}  \alpha t^{\alpha - 1} \Pr[L(h, z) > t] dt \right)^{\frac{1}{2}} \left( 
\frac{t_0}{\gamma^2 I_\alpha^{\frac{\alpha - 1}{\alpha}}} + \int_{t_0}^{+\infty}  \frac{1}{\alpha t^{\alpha - 1}} dt \right)^{\frac{1}{2}}\\
& \leq \left( \gamma^2 I_\alpha^{\frac{\alpha - 1}{\alpha}} t_0 + I_\alpha \right)^{\frac{1}{2}} \left( 
\frac{t_0}{\gamma^2 I_\alpha^{\frac{\alpha - 1}{\alpha}}} + \frac{1}{\alpha (\alpha - 2) t_0^{\alpha - 2}}
 \right)^{\frac{1}{2}}.
\end{aligned}
\end{multline*}
Introducing $t_1$ with $t_0 = I_\alpha^{1/\alpha} t_1$ leads to
\begin{align*}
\int_0^{+\infty}  \sqrt{\Pr[L(h, z) > t]} dt 
& \leq \left( \gamma^2 I_\alpha t_1 + I_\alpha \right)^{\frac{1}{2}} \left( 
\frac{t_1}{\gamma^2 I_\alpha^{\frac{\alpha - 2}{\alpha}}} + \frac{1}{\alpha (\alpha - 2) t_1^{\alpha - 2}  I_\alpha^{\frac{\alpha - 2}{\alpha}}}
 \right)^{\frac{1}{2}}\\
& \leq \left( \gamma^2  t_1 + 1 \right)^{\frac{1}{2}} \left( 
\frac{t_1}{\gamma^2} + \frac{1}{\alpha (\alpha - 2) t_1^{\alpha - 2}}
 \right)^{\frac{1}{2}} I_\alpha^{\frac{1}{\alpha}}.
\end{align*}
We now seek to minimize the expression $\left( \gamma^2 t_1 + 1
\right)^{\frac{1}{2}} \left( \frac{t_1}{\gamma^2} + \frac{1}{\alpha
    (\alpha - 2) t_1^{\alpha - 2}} \right)^{\frac{1}{2}}$, first as a
function of $\gamma$. This expression can be viewed as the product of
the norms of the vectors $\u = (\gamma t_1^{\frac{1}{2}}, 1)$ and $\v =
(\frac{t_1^{\frac{1}{2}}}{\gamma}, \frac{1}{\sqrt{\alpha (\alpha - 2)}
  t_1^{\frac{\alpha - 2}{2}}})$, with a constant inner product (not
depending on $\gamma$). Thus, by the properties of the Cauchy-Schwarz
inequality, it is minimized for collinear vectors and in that case
equals their inner product:
\begin{equation*}
\u \cdot \v = t_1 + \frac{1}{\sqrt{\alpha (\alpha - 2)}
  t_1^{\frac{\alpha - 2}{2}}}.
\end{equation*}
Differentiating this last expression with respect to $t_1$ and setting
the result to zero gives the minimizing value of $t_1$:
$(\frac{2}{\alpha - 2} \sqrt{\alpha (\alpha - 2)})^{-\frac{2}{\alpha}}
= \left( \frac{1}{2} \sqrt{\frac{\alpha - 2}{\alpha}}
\right)^{\frac{2}{\alpha}}$. For that value of $t_1$,
\begin{equation*}
\u \cdot \v = \left( 1 + \frac{2}{\alpha - 2} \right) t_1 = \frac{\alpha}{\alpha - 2} \left( \frac{1}{2} \sqrt{\frac{\alpha - 2}{\alpha}} \right)^{\frac{2}{\alpha}} = \left( \frac{1}{2} \right)^{\frac{2}{\alpha}}
\left( \frac{\alpha - 2}{\alpha} \right)^{\frac{1 - \alpha}{\alpha}},
\end{equation*}
which concludes the proof.
\end{proof}

\begin{theorem}
\label{th:main2}
Let $\alpha > 2$, $0 < \e \leq 1$, and $0 < \tau \leq \e^2$. Then, for
any loss function $L$ (not necessarily bounded) and hypothesis set $H$
such that $\L_\alpha(h) < +\infty$ and $\h \L_\alpha(h) < +\infty$
for all $h \in H$, the following two inequalities hold:
\begin{equation*}
\Pr \bigg[\sup_{h \in H} \frac{\L(h) - \h
  \L(h)}{\sqrt[\alpha]{\L_\alpha(h) + \tau}} > 
\Lambda(\alpha)  \e   \bigg]
\leq \Pr \bigg[\sup_{h \in H, t \in \Rset} \frac{ \Pr[L(h, z) > t] - \h \Pr[L(h, z) > t] }{\sqrt{\Pr[L(h, z) > t] + \tau}}  > \e \bigg]
\end{equation*}
\begin{equation*}
\Pr \bigg[\sup_{h \in H} \frac{\h \L(h) - \L(h)}{\sqrt[\alpha]{\h \L_\alpha(h) + \tau}} > 
\Lambda (\alpha)  \e   \bigg]
\leq \Pr \bigg[\sup_{h \in H, t \in \Rset} \frac{\h \Pr[L(h, z) > t] - \Pr[L(h, z) > t] }{\sqrt{\h \Pr[L(h, z) > t] + \tau}}  > \e \bigg],
\end{equation*}
where $\Lambda(\alpha) = \big( \frac{1}{2} \big)^{\frac{2}{\alpha}}
\big( \frac{\alpha}{\alpha - 2} \big)^{\frac{\alpha - 1}{\alpha}} +
\frac{\alpha}{\alpha - 1} \tau^{\frac{\alpha - 2}{2 \alpha}}$.
\end{theorem}

\begin{proof}
  We prove the first statement since the second one can be proven in a
  very similar way.  Assume that $\sup_{h, t} \frac{ \Pr[L(h, z) > t] - \h
    \Pr[L(h, z) > t] }{\sqrt{\Pr[L(h, z) > t] + \tau}} \leq \e$. Fix $h \in
  H$, let $J = \int_0^{+\infty} \sqrt{\Pr\left[ L(h, z) > t\right]} \, dt$
  and $\nu = \L_\alpha(h)$.  By Markov's inequality, for any $t >
  0$, $\Pr[L(h, z) > t] = \Pr[L^\alpha(h, z) > t^\alpha] \leq
  \frac{\L_\alpha(h)}{t^\alpha} = \frac{\nu}{t^\alpha}$.  Using this
  inequality, for any $t_0 > 0$, we can write
\begin{align*}
\L(h) - \h \L(h) 
& = \int_{0}^{+\infty} (\Pr[L(h, z) > t]  - \h \Pr[L(h, z) > t]) \, dt \\
& = \int_{0}^{t_0} (\Pr[L(h, z) > t]  - \h \Pr[L(h, z) > t]) \, dt + \int_{t_0}^{+\infty} (\Pr[L(h, z) > t]  - \h \Pr[L(h, z) >
  t]) \, dt\\
& \leq \e \int_{0}^{t_0} \sqrt{\Pr[L(h, z) >
    t] + \tau} \, dt + \int_{t_0}^{+\infty} \Pr[L(h, z) > t] \, dt\\
& \leq \e \int_{0}^{t_0} (\sqrt{\Pr[L(h, z) >
    t]} + \sqrt{\tau}) \, dt  + \int_{t_0}^{+\infty} \frac{\nu}{t^\alpha} \, dt\\
& \leq \e J  + \e \sqrt{\tau} t_0 + \frac{\nu}{(\alpha - 1)
  t_0^{\alpha - 1}}.
\end{align*}
Choosing $t_0$ to minimize the right-hand side yields $t_0 = \big(\frac{\nu}{\e
  \sqrt{\tau}} \big)^{\frac{1}{\alpha}}$ and gives
\begin{equation*}
\L(h) - \h \L(h) \leq \e J  + \frac{\alpha}{\alpha - 1}
\nu^{\frac{1}{\alpha}} (\e \sqrt{\tau})^{\frac{\alpha - 1}{\alpha}}.
\end{equation*}
Since $\tau \leq \e^2$, $(\e \sqrt{\tau})^{\frac{\alpha - 1}{\alpha}}
= [\e \tau^{\frac{1}{2 (\alpha - 1)}} \tau^{\frac{\alpha - 2}{2
    (\alpha - 1)}}]^{\frac{\alpha - 1}{\alpha}}
\leq [\e \e^{\frac{1}{(\alpha - 1)}} \tau^{\frac{\alpha - 2}{2
    (\alpha - 1)}}]^{\frac{\alpha - 1}{\alpha}} = \e
\tau^{\frac{\alpha - 2}{2 \alpha}}$.
Thus, by Proposition~\ref{prop:1}, the following holds:
\begin{equation*}
\frac{\L(h) - \h \L(h)}{\sqrt[\alpha]{\L_\alpha(h) + \tau}} 
\leq \e \Psi(\alpha) \frac{\nu^{\frac{1}{\alpha}}}{(\nu + \tau)^{\frac{1}{\alpha}}}  + \frac{\alpha}{\alpha - 1}\e
\tau^{\frac{\alpha - 2}{2 \alpha}} \frac{\nu^{\frac{1}{\alpha}}}{(\nu
  + \tau)^{\frac{1}{\alpha}}}
\leq \e \Psi(\alpha) + \frac{\alpha}{\alpha - 1}\e
\tau^{\frac{\alpha - 2}{2 \alpha}},
\end{equation*}
which concludes the proof.
\end{proof}
Combining Theorem~\ref{th:main2} with Theorem~\ref{th:relative} leads
immediately to the following two results.

\begin{corollary}
\label{cor:main2}
Let $\alpha > 2$, $0 < \e \leq 1$, and $0 < \tau \leq \e^2$. Then, for
any loss function $L$ (not necessarily bounded) and hypothesis set $H$
such that $\L_\alpha(h) < +\infty$ and $\h \L_\alpha(h) < +\infty$
for all $h \in H$, the following two inequalities hold:
\begin{align*}
& \Pr \bigg[\sup_{h \in H} \frac{\L(h) - \h
  \L(h)}{\sqrt[\alpha]{\L_\alpha(h) + \tau}} > 
\Lambda(\alpha)  \e   \bigg]
\leq 4 \, \E[\S_Q(z_1^{2m})] \exp \bigg(   \frac{-m \e^2 }{4}  \bigg)\\
& \Pr \bigg[\sup_{h \in H} \frac{\h \L(h) - \L(h)}{\sqrt[\alpha]{\h \L_\alpha(h) + \tau}} > 
\Lambda (\alpha)  \e   \bigg]
\leq 4 \, \E[\S_Q(z_1^{2m})] \exp \bigg(   \frac{-m \e^2 }{4}  \bigg),
\end{align*}
where $\Lambda(\alpha) = \big( \frac{1}{2} \big)^{\frac{2}{\alpha}}
\big( \frac{\alpha}{\alpha - 2} \big)^{\frac{\alpha - 1}{\alpha}} +
\frac{\alpha}{\alpha - 1} \tau^{\frac{\alpha - 2}{2 \alpha}}$ and
where $Q$ is the set of functions $Q = \set{z \mapsto 1_{L(h, z) > t}
  \mid h \in H, t \in \mathbb{R}}$.
\end{corollary}

In the following result, $\Pdim(G)$ denotes the pseudo-dimension of a
family of real-valued functions $G$
\citep{Pollard84,Pollard89,vapnik98}, which coincides
with the VC-dimension of the corresponding thresholded functions:
\begin{equation}
 \Pdim(G) = \text{VCdim}\Big(\big\{(x, t) \mapsto 1_{(g(x) - t)
> 0} \colon g \in G \big\}\Big) \,.
\end{equation}

\begin{corollary}
\label{cor:main21}
Let $\alpha > 2$, $0 < \e \leq 1$. Let $L$ be a loss function (not
necessarily bounded) and $H$ a hypothesis set such that $\L_\alpha(h)
< +\infty$ for all $h \in H$, and $d = \Pdim(\set{z \mapsto L(h, z)
  \mid h \in H}) < +\infty$. Then, for any $\delta > 0$, with
probability at least $1 - \delta$, each of the following inequalities
holds for all $h \in H$:
\begin{align*}
&  \L(h) \leq \h \L(h) + 2 \Lambda (\alpha) \sqrt[\alpha]{\L_\alpha(h)} \sqrt{\frac{d \log \frac{2em}{d}+ \log \frac{4}{\delta}}{m}}\\
& \h \L(h) \leq \L(h) + 2 \Lambda (\alpha) \sqrt[\alpha]{\h \L_\alpha(h)} \sqrt{\frac{d \log \frac{2em}{d}+ \log \frac{4}{\delta}}{m}}
\end{align*}
where $\Lambda(\alpha) = \big( \frac{1}{2} \big)^{\frac{2}{\alpha}}
\big( \frac{\alpha}{\alpha - 2} \big)^{\frac{\alpha - 1}{\alpha}}$.
\end{corollary}

\section{Conclusion}

We presented a series of results for relative deviation bounds 
used to prove generalization bounds for unbounded loss
functions. These learning bounds can be used in a variety of
applications to deal with the more general unbounded case. The
relative deviation bounds are of independent interest and can
be further used for a sharper analysis of guarantees in binary
classification and other tasks.

\ignore{
\section*{Acknowledgments}
}

\bibliography{ubound}

\begin{thebibliography}{39}
\providecommand{\natexlab}[1]{#1}
\providecommand{\url}[1]{\texttt{#1}}
\expandafter\ifx\csname urlstyle\endcsname\relax
  \providecommand{\doi}[1]{doi: #1}\else
  \providecommand{\doi}{doi: \begingroup \urlstyle{rm}\Url}\fi

\bibitem[Anthony and Shawe-Taylor(1993)]{AnthonyShawe-Taylor1993}
M.~Anthony and J.~Shawe-Taylor.
\newblock A result of {Vapnik} with applications.
\newblock \emph{Discrete Applied Mathematics}, 47:\penalty0 207 -- 217, 1993.

\bibitem[Anthony and Bartlett(1999)]{AnthonyBartlett99}
Martin Anthony and Peter~L. Bartlett.
\newblock \emph{Neural Network Learning: Theoretical Foundations}.
\newblock Cambridge University Press, 1999.

\bibitem[Azuma(1967)]{Azuma67}
Kazuoki Azuma.
\newblock Weighted sums of certain dependent random variables.
\newblock \emph{Tohoku Mathematical Journal}, 19\penalty0 (3):\penalty0
  357--367, 1967.

\bibitem[Bartlett and Mendelson(2002)]{BartlettMendelson2002}
Peter~L. Bartlett and Shahar Mendelson.
\newblock Rademacher and {G}aussian complexities: Risk bounds and structural
  results.
\newblock \emph{Journal of Machine Learning Research}, 3, 2002.

\bibitem[Bartlett et~al.(2002{\natexlab{a}})Bartlett, Boucheron, and
  Lugosi]{BartlettBoucheronLugosi2002}
Peter~L. Bartlett, St\'{e}phane Boucheron, and G\'{a}bor Lugosi.
\newblock Model selection and error estimation.
\newblock \emph{Machine Learning}, 48:\penalty0 85--113, September
  2002{\natexlab{a}}.

\bibitem[Bartlett et~al.(2002{\natexlab{b}})Bartlett, Bousquet, and
  Mendelson]{BartlettBousquetMendelson2002}
Peter~L. Bartlett, Olivier Bousquet, and Shahar Mendelson.
\newblock Localized {R}ademacher complexities.
\newblock In \emph{COLT}, volume 2375, pages 79--97. Springer-Verlag,
  2002{\natexlab{b}}.

\bibitem[Ben-David et~al.(2007)Ben-David, Blitzer, Crammer, and
  Pereira]{bendavid}
S.~Ben-David, J.~Blitzer, K.~Crammer, and F.~Pereira.
\newblock Analysis of representations for domain adaptation.
\newblock \emph{NIPS}, 2007.

\bibitem[Bickel et~al.(2007)Bickel, Br\"{u}ckner, and Scheffer]{bickel-icml07}
Steffen Bickel, Michael Br\"{u}ckner, and Tobias Scheffer.
\newblock Discriminative learning for differing training and test
  distributions.
\newblock In \emph{ICML}, pages 81--88, 2007.

\bibitem[Blitzer et~al.(2008)Blitzer, Crammer, Kulesza, Pereira, and
  Wortman]{blitzer}
J.~Blitzer, K.~Crammer, A.~Kulesza, F.~Pereira, and J.~Wortman.
\newblock Learning bounds for domain adaptation.
\newblock \emph{NIPS 2007}, 2008.

\bibitem[Boucheron et~al.(2005)Boucheron, Bousquet, and
  Lugosi]{BoucheronBousquetLugosi2005}
St\'ephane Boucheron, Olivier Bousquet, and G\'{a}bor Lugosi.
\newblock Theory of classification: a survey of recent advances.
\newblock \emph{ESAIM: Probability and Statistics}, 9:\penalty0 323--375, 2005.

\bibitem[Cortes and Mohri(2013)]{nsmooth}
Corinna Cortes and Mehryar Mohri.
\newblock Domain adaptation and sample bias correction theory and algorithm for
  regression.
\newblock \emph{Theoretical Computer Science}, 9474, 2013.

\bibitem[Cortes et~al.(2008)Cortes, Mohri, Riley, and Rostamizadeh]{bias}
Corinna Cortes, Mehryar Mohri, Michael Riley, and Afshin Rostamizadeh.
\newblock Sample selection bias correction theory.
\newblock In \emph{ALT}, 2008.

\bibitem[Cortes et~al.(2010)Cortes, Mansour, and Mohri]{importance}
Corinna Cortes, Yishay Mansour, and Mehryar Mohri.
\newblock Learning bounds for importance weighting.
\newblock In \emph{Advances in Neural Information Processing Systems (NIPS
  2010)}, Vancouver, Canada, 2010. MIT Press.

\bibitem[Dasgupta and Long(2003)]{DasguptaL03}
Sanjoy Dasgupta and Philip~M. Long.
\newblock Boosting with diverse base classifiers.
\newblock In \emph{COLT}, 2003.

\bibitem[{Daum\'e III} and Marcu(2006)]{daume06}
Hal {Daum\'e III} and Daniel Marcu.
\newblock Domain adaptation for statistical classifiers.
\newblock \emph{Journal of Artificial Intelligence Research}, 26:\penalty0
  101--126, 2006.

\bibitem[Dud\'ik et~al.(2006)Dud\'ik, Schapire, and Phillips]{dudik}
Miroslav Dud\'ik, Robert~E. Schapire, and Steven~J. Phillips.
\newblock Correcting sample selection bias in maximum entropy density
  estimation.
\newblock In \emph{NIPS}, 2006.

\bibitem[Dudley(1984)]{Dudley84}
R.~M. Dudley.
\newblock A course on empirical processes.
\newblock \emph{Lecture Notes in Mathematics}, 1097:\penalty0 2 -- 142, 1984.

\bibitem[Dudley(1987)]{Dudley87}
R.~M. Dudley.
\newblock Universal {Donsker} classes and metric entropy.
\newblock \emph{Annals of Probability}, 14\penalty0 (4):\penalty0 1306 -- 1326,
  1987.

\bibitem[Greenberg and Mohri(2013)]{greenbergmohri2013}
S.~Greenberg and M.~Mohri.
\newblock Tight lower bound on the probability of a binomial exceeding its
  expectation.
\newblock \emph{Technical Report 2013-957, Courant Institute, New York, New
  York}, 2013.

\bibitem[Haussler(1992)]{Haussler92}
David Haussler.
\newblock Decision theoretic generalizations of the {PAC} model for neural net
  and other learning applications.
\newblock \emph{Inf. Comput.}, 100\penalty0 (1):\penalty0 78--150, 1992.

\bibitem[Hoeffding(1963)]{Hoeffding63}
Wassily Hoeffding.
\newblock Probability inequalities for sums of bounded random variables.
\newblock \emph{Journal of the American Statistical Association}, 58\penalty0
  (301):\penalty0 13--30, 1963.

\bibitem[Huang et~al.(2006)Huang, Smola, Gretton, Borgwardt, and
  Sch{\"o}lkopf]{huang-nips06}
Jiayuan Huang, Alexander~J. Smola, Arthur Gretton, Karsten~M. Borgwardt, and
  Bernhard Sch{\"o}lkopf.
\newblock Correcting sample selection bias by unlabeled data.
\newblock In \emph{NIPS}, volume~19, pages 601--608, 2006.

\bibitem[Jaeger(2005)]{Jaeger2005}
Savina~Andonova Jaeger.
\newblock Generalization bounds and complexities based on sparsity and
  clustering for convex combinations of functions from random classes.
\newblock \emph{Journal of Machine Learning Research}, 6:\penalty0 307--340,
  2005.

\bibitem[Jiang and Zhai(2007)]{jiang-zhai07}
Jing Jiang and ChengXiang Zhai.
\newblock {Instance Weighting for Domain Adaptation in NLP}.
\newblock In \emph{ACL}, 2007.

\bibitem[Koltchinskii(2006)]{Koltchinskii2006}
V.~Koltchinskii.
\newblock Local rademacher complexities and oracle inequalities in risk
  minimization.
\newblock \emph{Annals of Statistics}, 34\penalty0 (6), 2006.

\bibitem[Koltchinskii and Panchenko(2000)]{KoltchinskiiPanchenko2000}
Vladimir Koltchinskii and Dmitry Panchenko.
\newblock Rademacher processes and bounding the risk of function learning.
\newblock In \emph{High Dimensional Probability II}, pages 443--459.
  Birkh\"{a}user, 2000.

\bibitem[Koltchinskii and Panchenko(2002)]{KoltchinskiiPanchenko2002}
Vladmir Koltchinskii and Dmitry Panchenko.
\newblock Empirical margin distributions and bounding the generalization error
  of combined classifiers.
\newblock \emph{Annals of Statistics}, 30, 2002.

\bibitem[Mansour et~al.(2009)Mansour, Mohri, and Rostamizadeh]{nadap}
Yishay Mansour, Mehryar Mohri, and Afshin Rostamizadeh.
\newblock Domain adaptation: Learning bounds and algorithms.
\newblock In \emph{COLT}, 2009.

\bibitem[McDiarmid(1989)]{McDiarmid89}
Colin McDiarmid.
\newblock On the method of bounded differences.
\newblock \emph{Surveys in Combinatorics}, 141\penalty0 (1):\penalty0 148--188,
  1989.

\bibitem[Meir and Zhang(2003)]{MeirZhang2003}
Ron Meir and Tong Zhang.
\newblock {Generalization Error Bounds for Bayesian Mixture Algorithms}.
\newblock \emph{Journal of Machine Learning Research}, 4:\penalty0 839--860,
  2003.

\bibitem[Pollard(1984)]{Pollard84}
David Pollard.
\newblock \emph{Convergence of Stochastic Processess}.
\newblock Springer, New York, 1984.

\bibitem[Pollard(1989)]{Pollard89}
David Pollard.
\newblock Asymptotics via empirical processes.
\newblock \emph{Statistical Science}, 4\penalty0 (4):\penalty0 341 -- 366,
  1989.

\bibitem[Sauer(1972)]{Sauer72}
Norbert Sauer.
\newblock On the density of families of sets.
\newblock \emph{Journal of Combinatorial Theory, Series A}, 13\penalty0
  (1):\penalty0 145--147, 1972.

\bibitem[Sugiyama et~al.(2008)Sugiyama, Nakajima, Kashima, von B\"unau, and
  Kawanabe]{sugiyama-nips2008}
M.~Sugiyama, S.~Nakajima, H.~Kashima, P.~von B\"unau, and M.~Kawanabe.
\newblock Direct importance estimation with model selection and its application
  to covariate shift adaptation.
\newblock In \emph{NIPS}, 2008.

\bibitem[Talagrand(1994)]{talagrand_ineq}
Michael Talagrand.
\newblock Sharper bounds for gaussian and empirical processes.
\newblock \emph{Annals of Probability}, 22\penalty0 (1):\penalty0 28--76, 1994.

\bibitem[Vapnik(1998)]{vapnik98}
Vladimir~N. Vapnik.
\newblock \emph{Statistical Learning Theory}.
\newblock John Wiley \& Sons, 1998.

\bibitem[Vapnik(2006{\natexlab{a}})]{Vapnik2006}
Vladimir~N. Vapnik.
\newblock \emph{Estimation of Dependences Based on Empirical Data}.
\newblock Springer-Verlag, 2006{\natexlab{a}}.

\bibitem[Vapnik(2006{\natexlab{b}})]{vapnik06}
Vladimir~N. Vapnik.
\newblock \emph{Estimation of Dependences Based on Empirical Data, second
  edition}.
\newblock Springer, Berlin, 2006{\natexlab{b}}.

\bibitem[Vapnik and Chervonenkis(1971)]{VapnikChervonenkis71}
Vladimir~N. Vapnik and Alexey Chervonenkis.
\newblock On the uniform convergence of relative frequencies of events to their
  probabilities.
\newblock \emph{Theory of Probability and Its Applications}, 16:\penalty0 264,
  1971.

\end{thebibliography}
\appendix

\section{Lemmas in support of Section~\ref{sec:relative}}

\begin{lemma}
\label{lemma:f}
  Let $1 < \alpha \leq 2$ and for any $\eta > 0$, let $f\colon (0,
  +\infty) \times (0, +\infty) \to \Rset$ be the function defined by
  $f\colon (x, y) \mapsto \frac{x - y}{\sqrt[\alpha]{x + y + \eta}}$.
Then, $f$ is a strictly increasing function of $x$ and a strictly
decreasing function of $y$.
\end{lemma}
\begin{proof}
  $f$ is differentiable over its domain of definition and for all $(x,
  y) \in (0, +\infty) \times (0, +\infty)$,
\begin{align*}
& \frac{\partial f}{\partial x}(x, y) 
= \frac{(x + y +
    \eta)^{\frac{1}{\alpha}} - \frac{x - y}{\alpha} (x + y +
    \eta)^{\frac{1}{\alpha} - 1}}{(x + y +
    \eta)^{\frac{2}{\alpha}}} 
= \frac{\frac{\alpha - 1}{\alpha} x + \frac{\alpha + 1}{\alpha} y + \eta}{(x + y +
    \eta)^{1 + \frac{1}{\alpha}}}  > 0\\
& \frac{\partial f}{\partial y}(x, y) 
= \frac{- (x + y +
    \eta)^{\frac{1}{\alpha}} - \frac{x - y}{\alpha} (x + y +
    \eta)^{\frac{1}{\alpha} - 1}}{(x + y +
    \eta)^{\frac{2}{\alpha}}} 
= -\frac{\frac{\alpha + 1}{\alpha} x + \frac{\alpha - 1}{\alpha}  y + \eta}{(x + y +
    \eta)^{1 + \frac{1}{\alpha}}}  < 0.
\end{align*}
\end{proof}

\end{document}